\documentclass[10.5pt,reqno]{amsart}

\usepackage{amsbsy}
\usepackage{amscd}
\usepackage{amsfonts}
\usepackage{amsgen}
\usepackage{amsmath}
\usepackage{amsopn}
\usepackage{amssymb}
\usepackage{amstext}
\usepackage{amsthm}
\usepackage{amsxtra}
\usepackage[all]{xy}
\usepackage{comment}
\usepackage{euscript}

\usepackage{csquotes} 
\usepackage{physics}

\usepackage{mathrsfs}
\usepackage{MnSymbol}
\usepackage{rotating}
\usepackage{url}

\usepackage{mathtools}

\usepackage{algpseudocode,algorithm}

\usepackage{ifthen}

\usepackage{graphicx}
\usepackage{hyperref}

\hypersetup{
	setpagesize=false,
	bookmarksnumbered=true,%
	bookmarksopen=true,%
	colorlinks=true,%
	linkcolor=blue,
	citecolor=red,
}

\theoremstyle{plain}
\newtheorem{thm}{Theorem}[section]
\newtheorem{prop}[thm]{Proposition}
\newtheorem{lem}[thm]{Lemma}
\newtheorem{cor}[thm]{Corollary}

\theoremstyle{definition}\newtheorem{defn}[thm]{Definition}
\newtheorem{rmk}[thm]{Remark}

\newtheorem{note}[thm]{Notation}

\numberwithin{equation}{section}

\newcommand{\bpara}[1]{\paragraph{\textbf{#1}}}

\renewcommand{\theta}{\vartheta}
\renewcommand{\phi}{\varphi}
\renewcommand{\epsilon}{\varepsilon}
\renewcommand{\subset}{\subseteq}

\renewcommand{\Im}{\mr{Im}}
\renewcommand{\Re}{\mr{Re}}

\newcommand{\mc}[1]{\mathcal{#1}}
\newcommand{\mf}[1]{\mathfrak{#1}}
\newcommand{\eu}[1]{\EuScript{#1}}
\newcommand{\mr}[1]{\mathrm{#1}}
\newcommand{\mbb}[1]{\mathbb{#1}}

\newcommand{\N}{\mathbb N}

\newcommand{\R}{\mathbb R}
\newcommand{\C}{\mathbb C}

\newcommand{\K}{\mathbb K}

\newcommand{\bfj}{\mathbf j}

\newcommand{\rvT}{T}

\newcommand{\rvW}{W}
\newcommand{\rvX}{X}

\newcommand{\rvZ}{Z}

\newcommand{\mcR}{\mathcal R}
\newcommand{\mcG}{\mathcal G}

\newcommand{\hilb}{\mc{H}}

\DeclareMathOperator*{\argmin}{arg\,min}

\DeclareMathOperator*{\minimi}{minimize}

\DeclareMathOperator*{\supp}{supp}
\DeclareMathOperator*{\id}{id}
\DeclareMathOperator*{\Real}{Re}
\DeclareMathOperator*{\Image}{Im}

\newcommand{\eps}{\epsilon}
\newcommand{\del}[2]{\partial_{#2} #1 }
\newcommand{\ccop}[2]{G_{#1}^{#2}}
\newcommand{\cc}[1]{G_{#1}}

\newcommand{\HP}{\mathbb{H}^+}
\newcommand{\HM}{\mathbb{H}^-}

\newcommand{\ESD}{\mathrm{ESD}}
\newcommand{\SPN}{\mathrm{SPN}}
\newcommand{\CW}{\mathrm{CW}}
\newcommand{\sa}{\mathrm{s.a.}}

\newcommand{\borel}{\mc{B}(\R)} 
\newcommand{\allmoments}{\mc{B}^\infty(\R)} 
\newcommand{\RVallmoments}{\mc{L}^{\infty -}}

\newcommand{\cptprob}{\mc{B}_{c}(\R)} 
\newcommand{\PDF}{\mathrm{PDF}_+} 

\newcommand{\oball}[1]{U(0,#1)} 

\newcommand{\sskip}{\smallskip}

\begin{document}

\newboolean{DVIPDF}
\setboolean{DVIPDF}{false}

\mathtoolsset{showonlyrefs=true}



\title[Cauchy Noise Loss]{
	Cauchy noise loss\\
	for stochastic  optimization of  random matrix models \\
	via free deterministic equivalents
}
\author{Tomohiro Hayase}
\address{Graduate School of Mathematical Sciences, University of Tokyo, 3-8-1 Komaba, Meguro-ku, Tokyo, 153-8914, Japan }
\email{\href{mailto:}{hayase@ms.u-tokyo.ac.jp}}
\date{\today}

\begin{abstract}
For random matrix models, the parameter estimation based on the traditional likelihood functions is not straightforward in particular when we have only one sample matrix.
We introduce a new parameter optimization method for random matrix models which works even in such a case.
The method is based on the spectral distribution instead of the traditional likelihood.
In the method, the Cauchy noise has an essential role because the free deterministic equivalent, which is a tool in free probability theory, allows us to approximate the spectral distribution perturbed by Cauchy noises by a smooth and accessible density function.

Moreover, we study an asymptotic property of determination gap, which has a similar role as generalization gap.
Besides, we propose a new dimensionality recovery method for the signal-plus-noise model, and experimentally demonstrate that it recovers the rank of the signal part even if the true rank is not small.
It is a simultaneous rank selection and parameter estimation procedure.
\end{abstract}


\maketitle
\tableofcontents

\section{Introduction}

\noindent Situations in many fields of research, such as digital
communications and statistics, can
be modeled with random matrices.  
The development of free probability theory (FPT for short) invented by Voiculescu \cite{voiculescu1992free} expands the scope of research of random matrices.
The free probability theory is an invaluable tool for describing the asymptotic behavior of many random matrices when their size is large.
For example, consider  a fixed multivariate polynomial $P$, independent random matrices $Z_1, \dots, Z_n$, and the following;
\begin{enumerate}
	\item deterministic matrices $A_1, \dots, A_m$,
	\item the empirical spectral distribution of $P(A_1, \dots, A_m, Z_1, \dots, Z_n)$.
\end{enumerate}
Then  FPT answers how to infer (2) from (1) for a wide class of  polynomials and random matrices.
However, little is known about its opposite direction; that is, how to know (1) from (2).
This direction is regarded as a statistical problem of how to estimate parameters of a random matrix model from observed empirical spectral distributions (ESD, for short).
Now, estimating parameters of a system involving random matrices appears in several fields of engineering such as signal processing and machine learning. 
Therefore, we are interested in finding a common framework to treat several random matrix models using their algebraic structures.

\sskip

\bpara{Likelihood or Spectral Distribution}
The maximal likelihood estimation, equivalently the minimizing empirical cross-entropy, is available in the case there are a large number of i.i.d.\,samples.
For random matrix models, the parameter estimation based on the traditional likelihood is not straightforward in particular when we have only one sample matrix.
For example, row vectors or column ones are not i.i.d\,family, then it is not clear that maximal likelihood estimation is applicable.
We introduce a new parameter optimization method of random matrix models which works even in such a case not based on the traditional likelihood, instead of based on the spectral distribution. 
The so-called self-averaging property, which is an almost-sure convergence of ESD of random matrices, is a key to our method. 
In order to find a reasonable objective function to estimate parameters of random matrices, we focus on the fact that the ESD of a random matrix model is approximated by a deterministic measure such that its Cauchy transform is accessible; it is a fixed point of a holomorphic mapping. We choose the deterministic measure by replacing the random matrix model with its \emph{free deterministic equivalent} (FDE for short, see Definition~\ref{defn:FDE}).

Based on the FDE, we introduce an objective function, which is an empirical cross-entropy defined as the following;
\begin{align}\label{align:intro-cce}
\theta \mapsto  \mbb{E}_{ \lambda \sim \nu, \ T \sim \mr{Cauchy}(0,\gamma)}\left[ -\log\left[ - \frac{1}{\pi}\Im G_{\mu^\Box(\theta)}( \lambda + T + i\gamma) \right] \right],
\end{align}
where $\nu$ is the spectral distribution of an observed sample matrix. 
In \eqref{align:intro-cce}, the random variables $\lambda$ and $T$ are  independent,  $\lambda$ is distributed with the spectral distribution $\nu$, and $T$ is distributed with the Cauchy distribution of scale $\gamma >0$.
In addition, the probability measure $\mu^\Box(\theta)$ is the deterministic one which approximates the ESD of the random matrix $W_\theta$, and $G_{\mu^\Box(\theta)}(z)$ $(z \in \C \setminus \R)$ is its  Cauchy transform. 
Note that the $\gamma$-slice $-\pi^{-1}\Im G(\cdot + i\gamma)$ of the Cauchy transform is a strictly positive density function on $\R$.

We choose this objective function because of the following reasons.
The first  one  is, as mentioned above that the Cauchy transforms of the ESD becomes accessible by using iterative methods
if we replace the random matrix model by its FDE.
More precisely, we choose a family of deterministic probability measures $\mu^\Box(\theta)$  $( \theta \in \Theta)$,  which approximates $\ESD( W_\theta)$,
and moreover we can compute $G_{\mu^\Box(\theta)}$ by iterative methods.
Note that the convergence of the iterative methods is rigorously proven (see Section~\ref{ssection:iterative}).
Besides, the gradient of each $\gamma$-slice is computable (see Section~\ref{ssection:gradient}) by the chain rule and the implicit differentiation.
The last reason is that $\gamma$-slice has enough information to
distinguish original measures (see Lemma~\ref{lem:key_lemma_cauchy}).

\sskip
\bpara{Compound Wishart and Signal-plus-Noise Models}
Compound Wishart matrices and signal-plus-noise matrices are typical classes of random matrices.
In this paper, we apply our methods to their families; compound Wishart models (CW model, for short) and signal-plus-noise models (SPN model).
Compound  Wishart matrices are introduced by Speicher \cite{speicher1998combinatorial}, which also appear as sample covariance matrices of correlated samplings.
Their modifications appear in the analysis of some statistical models  \cite{couillet2011deterministic, collins2013compound, hasegawa2013random, hayase2017free}.
See \cite{hiai2006semicircle}  for more detail.
The SPN model appears in the signal precessing  \cite{ryan2007free, hachem2012large, vallet2012improved}. 
The SPN model is also closely related with the probabilistic principal component analysis (see \cite{tipping1999probabilistic}), the matrix completion, the low-rank approximation, the reduced rank singular value decomposition, and the dimensionality recovery  \cite{nakajima2011theoretical, nakajima2013global, nakajima2015condition}. 

\sskip
\bpara{Dimensionality Recovery}
Consider a rectangular random matrix model 
\begin{align}\label{align:ipn-origin}
Y = A + \sigma Z,
\end{align}
where $Z$ is a $p \times d$ Ginibre matrix, $A \in M_{p,d}(\C)$, and $\sigma \in \R$.
The parameter $A$ expresses the signal part of the $Y$, and $\sigma$ does the noise power.
Its  likelihood  is  given by 
\begin{align}\label{align:likelihood-origin}
p(Y | A, \sigma) &=\frac{1}{(2\pi\sigma^2/d)^{pd/2}}\exp{-\frac{d}{2\sigma^{2}}\Tr[(Y- A)^*(Y-A)]}.
\end{align}
Hence if $\sigma$ is fixed, for a sample matrix $D$, its maximal likelihood estimation is equivalent to the traditional trace norm minimization; 
\begin{align}\label{align:minimize-frob}
\minimi_{ A \in \Theta} \Tr[( D - A)^*( D-A)],
\end{align}
where $\Theta$ is a parameter space which is a subset of $M_{p,d}(\C)$.
For a  fixed $q \in \N$ and $\Theta = \{ A \mid \rank A \leq q\}$,
its closed-form solution is given by a well-known truncated singular value decomposition of $D$ (note that the field $\C$ can be replaced with $\R$), that is, given by replacing smaller $\min(p,d)-q$ singular values of $D$ with $0$.
Now, if the assumption  $\rank A \leq  q$ is removed, the solution is trivial; $A$ is estimated as the observed sample matrix itself.    
After all, for the low-rank approximation, we need to know the rank of the true parameter $A$ by another method beforehand if we use the likelihood function \eqref{align:likelihood-origin}.

Instead of the likelihood \eqref{align:likelihood-origin},  we apply our parameter estimation based on \eqref{align:intro-cce} to the low-rank approximation without the assumption on the true rank.
Here we focus on the empirical singular values of the large dimensional $Y$, equivalently, the empirical spectral distribution of  the signal-plus noise model defined as 
\[
W_\SPN(A,\sigma) := (A + \sigma Z)^*(A + \sigma Z).
\]
We emphasize that we estimate not only the signal part $A$ but also the noise power $\sigma$ from each single-shot sample matrix
by our new method.
\sskip

\bpara{Free Deterministic Equivalents}
Our work relies on the \emph{free deterministic equivalent} (FDE for short) introduced by  Speicher and Vargas \cite{speicher2012free, calros2015free}.
Roughly speaking,  we can interpret independent random matrices as deterministic matrices of operators in an infinite-dimensional $C^*$-probability space.
The origin of FDE can be found in Neu-Speicher~ \cite{neu1995rigorous} as a mean-field approximation of an Anderson model in statistical physics.
One can also consider FDE as one of the approximation methods to eliminate randomness for computing the expectation.
Particularly, FDE is a  \enquote{lift} of  the \emph{deterministic equivalent} introduced by \cite{hachem2007deterministic}.
More precisely, FDE is an approximation of a random matrix model at the level of operators, and on the other hand, the deterministic equivalent is that at the level of Cauchy transforms.
Now the deterministic equivalent is known as an approximation method
of Cauchy transforms of random matrices in several works of literature of wireless-network (see \cite{hachem2007deterministic, couillet2011deterministic}.
Despite its rich background in FPT, the algorithm of FDE is not complicated.
Roughly speaking, its primary step is to replace each Gaussian random variable in entries of a random matrix model by an \enquote{infinite size} Ginibre matrix, which is called a circular element in FPT.

As mentioned above, the Cauchy transform is accessible; which is given by the two iterative methods based on Helton-Far-Speicher \cite{helton2007operator} and Belinschi-Mai-Speicher \cite{belinschi2013analytic}. 
Note that analytical computations of Cauchy transforms are unknown for many random matrices.

\sskip 
\bpara{Our Contribution}

Here we summarize our contributions.

Our major contribution is to introduce a common framework for the parameter optimization of random matrix models,  which is a combination of the Cauchy noise loss, FDE, iterative methods for computing Cauchy transforms, and a stochastic gradient descent method.

The second one is to give a brief, and general computing method of gradients of Cauchy transforms of FDE, in particular, give a norm estimation of derivations of implicit functions, which appear in the iterative method for computing Cauchy transforms.

The third one is to show the asymptotic properties of the gap between the Cauchy cross-entropy and the empirical one.

The fourth one is to show optimizations of the CW model and the SPN model via the Cauchy noise loss experimentally.

The last one is to propose a new dimensionality recovery method for the signal-plus-noise model, and experimentally demonstrate that it recovers the rank of the signal part even if the true rank is not small.
It is a simultaneous rank selection and parameter estimation procedure.

\section{Related Work}

\noindent There are several applications of deterministic equivalents and FDE  to the analysis of multi-input multi-output channels \cite{couillet2011deterministic, speicher2012free}.

Ryan \cite{ryan2007free}  applied the \emph{free deconvolution} to SPN models. Their method is based on evaluating the difference of moments, \emph{the mean square error of moments}.
Since it uses an only finite number of lower-order moments, the error has subtotal information of the empirical distribution.  On the contrary, our method uses full information of the empirical distribution.

There are applications of the fluctuation of FDE  to some autoregressive moving-average models \cite{hasegawa2017fluctuations, hayase2017free}.
Their methods are based on the fluctuation of the CW model, and focus on the good-of-fit test of the parameter estimation, not for the parameter estimation itself.

\sskip
Another direction to the low-rank approximation is the Bayesian matrix factorization.
The matrix factorization model is defined as the following;  fix $p'\in  \N$ with $p' \leq p,d$ and factorize $A = A_1 A_2$ with $A_1 \in M_{p,p'}(\C), A_2 \in M_{p', d}(\C)$;
\[
Y = A_1A_2 + \sigma Z.
\]
In addition, use the likelihood  \eqref{align:likelihood-origin} and the Gaussian prior on $A_1$, $A_2$. The parameter $A$ is  estimated as the integration of $A_1A_2$ with the posterior distribution.  The hyperparamers of the prior distributions are determined by minimizing the Bayesian free energy, which is called empirical Bayesian matrix factorization.
See \cite{nakajima2011theoretical} for the theoretical analysis.
Tipping-Bishop \cite{tipping1999probabilistic} treats the case $A_2$ is known.

The variational Bayesian method (see \cite{christopher2016pattern}), which approximates posterior distributions, is also called a mean-field approximation in the Bayesian framework.
Recall that an origin of FDE  is a  mean-field approximation, but it is a  deterministic approximation of the empirical spectral distribution, which is different from the variational Bayesian method.

Nakajima-Sugiyama-Babacan-Tomioka \cite{nakajima2013global} and \cite{nakajima2015condition} gave the global analytic optimal solution of the empirical variational Bayesian matrix factorization (EVBMF, for short), and used it to a dimensionality recovery problem. 
Note that EVBMF almost surely recovers the true rank in the large scale limit under some assumptions \cite[Theorem~13, Theorem~15]{nakajima2015condition}, in particular, if the true rank is low.
Their loss function is based on the likelihood.
Recall that we use another loss function not based on the likelihood \eqref{align:likelihood-origin}. Note that in our method, we need no assumption on the true rank.

\section{Random Matrix Models}

\noindent In this section, we introduce random matrix models and our main idea.
\sskip
\bpara{Basic Notation}
In this paper, we fix a probability space $(\Omega, \mf{F}, \mbb{P})$. A  random variable (resp.\,real random variable) $\rvX$ is  a $\C$-valued (resp.\,$\R$-valued) Borel measurable function on the probability space.

\begin{enumerate}
	\item $\mbb{E}[\rvX] := \int \rvX(\omega) \mbb{P}(d\omega)$ for any integrable or nonnegative real random variable $\rvX$.
	\item $\mbb{V}[\rvX] := \mbb{E} [ \rvX^2 ] - \mbb{E}[\rvX ]^2$   for a square-integrable real random variable $\rvX$.
	
	\item $C(\R):= \{f \colon \R\to \R \mid \text{continuous}\}$.
	
	\item $L^1(\R):= \{f \colon \R\to \R \mid \text{Borel \ mesuarable and  Lebesgue intergrable}\}$.
	
\end{enumerate}

\subsection{Gaussian Random Matrix models}
\hfill

\begin{note}
	Let $\K$ be $\R$ or $\C$, and $p, d \in \N$.
	Let us denote by $M_{p, d}(\K)$  the set of $p \times d$ rectangular matrices over $\K$.
	We write  $M_{d}(\K):= M_{d,d}(\K)$.
	A random matrix is a map $\Omega \to M_{p,d}(\C)$ for a $p, d \in \N$ such that each entry is Borel measurable.
\end{note}

\begin{defn}
	\hfill
	\begin{enumerate}
		\item A \emph{real Ginibre matrix} of size $p \times d$ is the $p \times d$ matrix whose entries are independent and identically distributed with $\mr{Normal}(0,v)$ for a $v>0$.
		
		We denote by $\mr{GM}(p,d,\R)$ the set of real Ginibre matrices with $v = 1/d$.
		
		\item A \emph{complex Ginibre matrix}
		of size $p \times d$ is the $p \times d$ matrix whose entries are given by $(1/\sqrt{2}) (f_{k\ell} + g_{k\ell}\sqrt{-1})$, where the family $\bigcup_{k=1, \dots, p, \ell = 1, \dots, d}\{f_{k\ell}, g_{k\ell}\}$ is independent and each element is distributed with $\mr{Normal}(0, v)$ for $v>0$.
		
		We denote by $\mr{GM}(p,d,\C)$ the set of complex Ginibre matrices with $v = 1/d$.
	\end{enumerate}
	
\end{defn}

\begin{note}
	We write $\RVallmoments:= \{ \rvX : \Omega \to \C \mid \text{Borel measurable, $\mbb{E}[ |\rvX|^k] < \infty, \ k \in \N$} \}$.
\end{note}

Note that $\mr{GM}(p,d,\K) \subset M_{p,d}(\RVallmoments)$.

\begin{defn}
	Let  $P(x_1, x_2, \dots, x_{m+n}) := P(x_1, x_2, \dots, x_{m+n}, x_1^*, x_2^*, \dots, x_{m+n}^*) $ be a self-adjoint polynomial (that is, it is stable under replacing $x_j$ by $x_j^*$) in non-commutative  dummy variables $x_1, x_2, \dots, x_{m+n}$ and their adjoint $x_1^*, \dots, x_{m+n}^*$.
	Let   $\mf{I} := \{  (r_k,\ell_k)_{ k=1, \dots, m} , (p_k,d_k)_{ k=1, \dots, n} \}$ be a family of pairs of natural numbers and $P_\mf{I}$ be   corresponding  evaluation of $P$ defined as  
	\[
	P_\mf{I} \colon \prod_{k=1}^m M_{r_k, \ell_k}(\RVallmoments) \times  \prod_{k=1}^{n}M_{p_k, d_k}(\C) \to M_{p,d}(\RVallmoments),
	\]
	where products and sums satisfy dimension compatibility.
	
	Then the \emph{real (resp.\,complex) polynomial Ginibre matrix model} (\emph{PGM model}, for short)  of type $(P,\mf{I})$ on a subset $\Theta \subset \prod_{k=1}^n M_{p_k, d_k}(\K)$ with $\K=\R$ (resp.\,$\K=\C$) is the restriction of the following map $\bar{P}_\mf{I}$  to $\Theta$;
	\[
	\bar{P}_\mf{I} \colon \prod_{k=1}^n M_{p_k, d_k}(\K) \to M_{p,d}(\RVallmoments), \ 
	\bar{P}_\mf{I} ( D_1, \dots, D_n) := P_\mf{I}(\rvZ_1, \dots, \rvZ_m, D_1, \dots, D_n),
	\]
	where the family of  $Z_j \in \mr{GM}(r_k, \ell_k, \K)$ ( $j=1, \dots, m$) is independent.
	
\end{defn}

We introduce examples of Ginibre matrix models which are in the scope of our numerical experiments.

\begin{defn}\label{defn_original_model}\hfill
	\begin{enumerate}
		
		\item  A \emph{compound Wishart  model (\emph{CW} model for short)  of type $(p, d)$ on $\Theta \subset M_p(\C)_\sa$}, denoted by $W_\CW$, is the  PGM model of type $(P, ((p,d), (d,d))$ on the subset $\Theta$, where
		$P(x, a) = x^*ax$.
		Note that
		\[
		\rvW_\CW(d, A) = \bar{P}_\mf{J}(A) = \rvZ^* A \rvZ.
		\]

		\item A \emph{signal-plus-noise model} (\emph{SPN} model for short) of type $(p,d)$ on a subset  $\Theta \subset M_{p,d}(\C) \times \R$, denoted by $\rvW_\SPN$,  is the PGM model of type $\left(P, \left(\left(p,d\right), \left(p,d\right), \left(1,1\right)\right) \right)$ on $\Theta$, where $P$ is a polynomial of dummy variables $x,a,\sigma$ given by
		$
		P (x, a, \sigma) = (x + \sigma a)^*(x + \sigma a).
		$
		Note that
		\[
		\rvW_\SPN(A, \sigma) = \bar{P}_\mf{J}(A,\sigma) = (A+ \sigma \rvZ)^*(A+ \sigma \rvZ),
		\]
	\end{enumerate}
\end{defn}

Our estimation method is based on the spectral distribution of random matrices.

\begin{defn}(Spectral distribution and Moments)
	
	\begin{enumerate}
		\item For any self-adjoint matrix $A \in M_d(\C)$, let $ \lambda_1\leq \lambda_2 \leq \dots \lambda_d$  be the eigenvalues of $A$. The \emph{spectral distribution of $A$}, denoted by $\mu_A$, is defined as the discrete measure 
		\[
		\mu_A = \frac{1}{d} \sum_{k=1}^d \delta_{\lambda_k}.
		\]
		
		\item For any self-adjoint random matrix $W: \Omega \to  M_d(\C)_\sa$, we write 
		\[
		\ESD(W)(\omega) := \mu_{W(\omega)}, \ \omega \in \Omega.
		\]
		
		\item
		Write $\borel := \{ \text{Borel probability measures on $\R$}\}$ and $\allmoments:= \{ \mu \in \borel \mid \int \abs{x}^k\mu(dx) < \infty \text{ for any } k \in \N\}$.

		For any $\mu \in \borel$,  we denote  the $k$-th moment of $\mu$ for $k \in \N$ by
		$m_k(\mu):= \int x^k \mu(dx)$.
		For any random variable $\rvX$ whose law is $\mu$, we define its moment by  $m_k(\rvX) := m_k(\mu)$.
		All moments of a probability measure (resp. a random variable ) are well-defined if $\mu \in \allmoments$ (resp.\,$\rvX \in \RVallmoments$).
		
		\item     $\Tr(A) := \sum_{k=1}^{d}A_{k,k}$ for $A \in M_d(\K)$, and $\tr := d^{-1}\Tr$.
		
		\item We  define the moment of $A$ as  $m_k(A) = \tr(A^k)$. 
	\end{enumerate}
	
\end{defn}

Note that $m_k(A) = m_k(\mu_A)$ for any $A \in M_d(\C)$ and $d \in \N$.

\subsection{From Cauchy Transform to  Cauchy Noise Loss}\label{subsection_Cauchy}\hfill

\noindent 
We use the Cauchy transform of ESD to define our loss function for parameter estimation of random matrix models.
This is mainly because the Cauchy transform of ESD is accessible for specific random matrix models.
More precisely, the Cauchy transform is approximated by the unique solution of a fixed point formula. Note that the fixed point formula depends on random matrix models.
We discuss the fixed point formula in Section~\ref{ssection:iterative}.
Besides, the Cauchy transform is closely rated with the Cauchy noise.

\begin{defn}\hfill
	\begin{enumerate}
		\item
		\emph{The Cauchy transform of  $\mu \in \borel$} is the holomorphic function $\cc{\mu}$ on $\C\setminus\R$ defined as
		\[
		\cc{\mu}(z) : =  \int \frac{1}{z -t }\mu(dt).
		\]
		
		\item The \emph{Cauchy distribution}  with scale parameter $\gamma > 0 $  is the probability measure over $\R$ whose density function is given by the following \emph{Poisson kernel};
		\[
		P_\gamma(x) := \frac{1}{\pi}\frac{\gamma}{x^2 + \gamma^2}, \ x \in \R.
		\]
		We call  a random variable $\rvT$  a \emph{Cauchy noise of scale $\gamma$}, denoted by $ \rvT \sim \mr{Cauchy}(0,\gamma)$, if its  density function  is equal to $P_\gamma$.
		
	\end{enumerate}
	
\end{defn}
\smallskip

The following is a key lemma of our algorithm.

\begin{defn}
	For  $f \in L^1(\R)$ and $\mu \in \borel$, we define their convolution $f*\mu \in L^1(\R)$ as $f*\mu(x) := \int f(x-t)\mu(dt)$, $x \in \R$.
\end{defn}

\begin{lem}\label{lem:key_lemma_cauchy}
	
	Let $\mu, \nu \in \borel$.    Fix $\gamma >0$. Then the following conditions are equivalent.
	\begin{enumerate}
		\item[(1)] For any $x \in \R$, 
		\begin{align}\label{align:key-formula}
		- \frac{1}{\pi}\Im G_\mu(x + i\gamma) = P_\gamma * \nu(x).
		\end{align}
		\item[(1')] $P_\gamma*\mu = P_\gamma*\nu$.
		\item[(2)] $\mu = \nu$.            
	\end{enumerate}
\end{lem}

\begin{proof}
	The equivalence of (1) and (1') follows from the well-known fact that
	\begin{align}
	- \frac{1}{\pi}\Im G_\mu(x + i\gamma) = P_\gamma * \mu(x).
	\end{align}
	
	Then we only need to show that (1') induces (2). 
	For any $\mu \in \borel$, let us denote its Fourier transform by $\mathcal{F}(\mu)(\xi) := \int \exp(ix\xi)\mu(dx)$.
	Similarly, we define the Fourier transform of any probability density function.
	Fix $\gamma >0$. Assume that $P_\gamma*\mu = P_\gamma*\nu$.
	Since the Fourier transformation linearize the convolution, we have  $\mathcal{F}(P_\gamma)\mathcal{F}(\mu) = \mathcal{F}(P_\gamma)\mathcal{F}(\nu)$.
	Because  $\mathcal{F}(P_\gamma)(\xi) =  \exp(-\gamma \abs{\xi}) > 0 \ (\xi \in\R)$, we have $\mathcal{F}(\mu) = \mathcal{F}(\nu)$. Since the Fourier transformation is injective, we have
	$\mu = \nu$.
\end{proof}

Let  $(W_\theta)_{\theta \in \Theta}$ be a family of a self-adjoint random matrix.
Then, for random matrix models such as CW and SPN, there is a family of deterministic measures $(\mu^\Box(\theta))_{\theta \in \Theta}$ which approximates $\ESD(W_\theta)$ (the proof is postponed to Section~\ref{ssec:fde}).
We estimate parameters by comparing a deterministic measure and a random one instead of comparing two random measures;
recall that there is only one single-shot observation.

Fix $\theta_0$, pick $\omega \in \Omega$ and let $\nu = \ESD(W_{\theta_0})(\omega)$.
Then the right-hand side of \eqref{align:key-formula} is equal to  the density  of the real random variable $\lambda+ T$, where $\lambda$ is a random variable distributed with 
the empirical spectral distribution, $T$ is a Cauchy noise of scale $\gamma$, and the pair is independent.
On the other hand,  each Cauchy transform $G_{\mu_\theta^\Box}$ approximates that of  $\ESD(W_\theta)$, and it is accessible; it is given by the solution of a fixed point formula and computed by an iterative method.
Then Lemma~\ref{lem:key_lemma_cauchy} suggests a possibility of the parameter estimation by fitting parametric implicit density functions to an empirical distribution perturbed by Cauchy noises.

\begin{defn}\label{defn:gamma-slice}
	Let $\mu \in \borel$.
	For $\gamma >0$, we call the strictly positive function $ x \in \R \mapsto - \pi^{-1}\Image G_\mu(x+i\gamma) \in \R$ the \emph{$\gamma$-slice} of $\mu$. If there is no confusion, we also call it the $\gamma$-slice of the Cauchy transform $G_\mu$.
\end{defn}

Now, in many kinds of research of statistics, the cross-entropy is used for fitting a density function to a reference distribution.
To achieve optimization of random matrix models,  we consider the following Cauchy cross-entropy.

\begin{defn}(Cauchy cross-entropy)\label{defn:cce}
	Let $\mu \in \cptprob$ and $\nu \in \borel$ with $m_2(\nu) < \infty$.
	Then the \emph{Cauchy cross-entropy} of $\mu$ against $\nu$ is defined as 
	\begin{align}\label{align:direct-rep-cce}
	H_\gamma( \nu, \mu) : = \int \ell_\gamma(x,\mu) P_\gamma*\nu(x)dx,
	\end{align}    
	where 
	\begin{align}
	\ell_\gamma(x,\mu) :=-\log \left[  - \frac{1}{\pi} \Image \cc{\mu}(x + i\gamma) \right], \ x \in \R.
	\end{align}
	
\end{defn}

\begin{rmk}
	We prove that the Cauchy cross-entropy is well-defined and finite in Section~\ref{sssec:basic-properties}.
	Note that 
	\begin{align}\label{align:symmetric-rep-cce}
	H_\gamma( \nu, \mu)  = H(P_\gamma*\nu, P_\gamma*\mu), 
	\end{align}
	where $H(q,p) := \int  -q(x)\log p(x) dx$ is the cross-entropy of a probability density function $p$ against $q$. 
	We use both representations \eqref{align:direct-rep-cce} and \eqref{align:symmetric-rep-cce} of the Cauchy  cross-entropy in the later sections.
	
	Recall that the cross-entropy is possibly ill-defined, in particular, if $p$ and $q$ have disjoint compact supports.
	We emphasize that the ESD of random matrices are approximated by compactly supported probability measures, which can have singular parts. Hence it is difficult for the usual cross-entropy $H(\cdot, \cdot)$ to treat ESD.
	However, the Cauchy cross-entropy  $H_\gamma(\cdot,\cdot)$ is well-defined even if the measures have compact supports.
	
\end{rmk}

We have \emph{the principal of minimum Cauchy cross-entropy} as follows.
\begin{prop}\label{prop:principal-minimum-cc}
	Fix $\gamma > 0$.
	For any $\nu \in \mathcal{B}_c(\R)$, it holds that
	\[  \argmin_{\mu \in \mathcal{B}_c(\R)}  H_\gamma(\nu , \mu) =\{\nu \}.
	\]
\end{prop}
The proof is postponed to Section \ref{sssec:basic-properties} after proving Lemma~\ref{lemma:cne-well-defined}.

According to the principal of minimal Cauchy cross-entropy, we consider the following minimizing problem  for a fixed $\nu$:
\begin{align}
\minimi_{ \theta \in \Theta} H_\gamma(\nu, \mu_\theta^\Box),
\end{align}
where $(\mu_\theta^\Box)_{\theta \in \Theta}$ is the family of deterministic probability measures  mentioned above, which approximate $(\ESD(W_\theta))_{\theta \in \Theta}$.
In our setting, we assume that $\nu$ is a single-shot observation of empirical distribution for an unknown parameter $\theta_0 \in \Theta$; $\nu = \ESD(W_{\theta_0})(\omega)$. That is, we consider 
\begin{align}\label{align:empirical-cauchy-cross-entropy-minimization}
\minimi_{ \theta \in \Theta} H_\gamma( \ESD(W_{\theta_0})(\omega), \mu_\theta^\Box).
\end{align}
Furthermore we show  in Section~\ref{sssec:determination-gap} that the gap between $H_\gamma(\mu_{\theta_0}^\Box, \mu_\theta^\Box)$ and  $H_\gamma( \ESD(W_{\theta_0}), \mu_\theta^\Box)$ is small uniformly on the parameter space, and almost surely on the probability space.

\begin{defn}(Empirical Cauchy cross-entropy)
	Let $W$ be a self-adjoint random matrix and $\mu \in \cptprob$.
	Fix $\gamma >0$.
	Then the     \emph{empirical Cauchy cross-entropy} of $\mu$  against $W$ is defined as
	the real random variable $H_\gamma( \ESD(W), \mu)$.
\end{defn}

Next, there are two key points to minimize the empirical risk; $\ell_\gamma(\cdot, \mu_\theta^\Box)$ is accessible, and the empirical Cauchy cross-entropy is written as the following expectation. 

\begin{lem} \label{lem:cce-expectational-rep}
	Under the setting of Definition~\ref{defn:cce}, it holds that
	\begin{align}
	H_\gamma( \nu, \mu) = \mbb{E}[ \ell_\gamma( \lambda + T, \mu)],
	\end{align}
	where $\lambda, T$ are independent real random variables such that $\lambda \sim \nu$ and $T \sim \mr{Cauchy}(0,\gamma)$. 
\end{lem}

\begin{proof}
	This follows from the fact that the density of $\lambda + T$ is $P_\gamma*\nu$.    
\end{proof}

Lemma~\ref{lem:cce-expectational-rep} is the key since there are many  stochastic approaches  to solve  minimization problem of the form
\[\minimi_{\theta \in \Theta} \mbb{E}[f(\zeta, \theta)], \]
under  a given parametric function $f( \cdot, \theta)$ and a random variable $\zeta$.  Robbins-Monro \cite{robbins1951stochastic} is their origin.
The online gradient descent iteratively updates parameters $\theta$ of the model based on the gradient at a sample $\zeta_t$ randomly picked from the total one at each iteration: $\theta_{t+1}  = \theta_{t} - \eta_t \grad_\theta  |_{\theta = \theta_t} f(\zeta_t, \theta)$.
There are several versions of the online gradient descent, see Algorithm~\ref{alg:BOGD-FDE} for the detail.

Here we introduce the Cauchy noise loss;
\begin{defn}\label{defn:cnl}(Cauchy Noise Loss)
	For $\mu \in \cptprob$ and $W$ be a self-adjoint random matrix.
	Let $d \in \N$ and $\lambda = (\lambda_1, \dots, \lambda_d) \in \R^d$.
	Then the \emph{Cauchy noise loss} is defined as the random variable
	\begin{align}
	L_{\gamma}(\lambda, \mu) := \ell_\gamma( \lambda_{\bfj} + T, \mu),
	\end{align}    
	where $\bfj$ is a uniform random variable  on $\{ 1, \dots, d\}$ and  $T$ is a Cauchy noise of scale $\gamma$ which is independent from $\bfj$.    
\end{defn}

\section{Algorithm}

\subsection{Loss Function and Its Gradient}
\hfill
\newcommand{\RRN}{\mr{RRN}}

\noindent Infinite-dimensional operators have theoretically critical roles in our methods.
However, to implement our algorithm, infinite-dimensional operators are not needed; our method works only using finite-dimensional ones.
For the reader's convenience, we summarize the results needed for the algorithm.

\begin{rmk}\label{rmk:reduce_model}
	Let $A = U D V $ be a singular value decomposition of $A$, where $U$ is a $p \times p $ unitary matrix, $V$ is a $d \times d$ unitary matrix, and $D$ is a $p \times d $ rectangular diagonal matrix. 
	Then $\rvW_\SPN(A, \sigma) = V^*(D + \sigma V\rvZ^*U )(D + \sigma U\rvZ V^*)V$. Hence $\ESD(\rvW_\SPN(A, \sigma))= \ESD(\rvW_\SPN(D, \sigma))$, since the joint distribution of the entries of $U\rvZ V$ and that of $\rvZ$ is same.
	Similarly, it holds that $\ESD( W_\CW(B)) = \ESD(\rvW_\CW(D))$, where $B = UDU^*$ is a diagonalization of $B$.
	Hence in the parameter estimation of CW or SPN models from each empirical spectral distribution, we cannot know such unitary matrices.
	Therefore, we consider the following restricted domains of parameters.
\end{rmk}

\begin{defn}\label{defn:domains}
	Let $p,d \in \N$ with $p \geq d$.
	Let us define
	\begin{align}
	\Theta_\CW(p,d) &:=  \{d\} \times M_{p}(\C)_\sa ,\\
	\Theta_\SPN(p,d) &:= M_{p,d}(\C) \times \R.
	\end{align}
	For $M > 0$, we define
	\begin{align}
	\Theta_\CW(p, d, M) &:=  \{ (d,V) \in \Theta_\CW(p,d)\mid   \norm{V} \leq M \},\\
	\Theta_\SPN(p,d, M) &:=  \{  (A, \sigma) \in M_{p,d}(\C) \times \R \mid  \norm{A^*A} \leq M^2 , \abs{\sigma} \leq M\}.
	\end{align}
	In addition, we define
	\begin{align}
	\Xi_\CW(p,d, M) &:= \{d\}\times \{ v \in \R^p  \mid   \norm{v}_\infty \leq M \},\\
	\Xi_\SPN(p,d, M) &:= \{ ( a, \sigma)  \in \R^d \times \R \mid  \norm{a}_\infty \leq M, \abs{\sigma} \leq M\}.
	\end{align}
	where $\norm{a}_\infty = \max_{d \in \R} \abs{a_d}$ for $a \in \R^d$ and  $d \in \N$.
	Let $m \in \N$ with $m \leq  p ,d$.
	We denote by $\iota^m_{p,d}$ the diagonal embedding   $\iota^m_{p,d} \colon \R^m \to M_{p,d}(\C)$ defined as
	\begin{align}
	\iota^m_{p,d}(a)_{ij} :=  \begin{cases}
	\delta_{ij}a_{j}, & \mr{if \ } i,j \leq m,\\
	0, & \text{otherwise}.
	\end{cases}
	\end{align}
	We write $\iota^m_d := \iota^m_{d,d}$ and $\iota_d := \iota^d_d$.
	To abuse the notation,  we write 
	\begin{align}
	W_\CW(v) &:=  W_\CW( d, \iota_p(v)),  \ d \in \N,   v \in \R^p,\\
	W_\SPN(a, \sigma) &:= W_\SPN(\iota^d_{p,d}(a), \sigma), \ a \in \R^d, \sigma \in \R.
	\end{align}
\end{defn}

In Definition~\ref{defn:domains}, note that $\iota_d$ maps $\Xi_\CW(p, d, M)$ into $\Theta_\CW(p, d, M)$ and $\iota^d_{p,d}$ maps $\Xi_\SPN(p,d, M)$ into $\Theta_\SPN(p,d, M)$.

\begin{note}\hfill
	\begin{enumerate}
		\item $\HP := \{ z \in \C \mid \Im z > 0\}$, and $\HM := - \HP$.
		\item $\HP(\C^2) := \{ Z \in \C^2 \mid \Im Z_1, \Im Z_2 > 0\}$, $\HM(\C^2) := - \HP(\C^2)$.
	\end{enumerate}
\end{note}

\begin{thm}\label{thm:main_CW}
	There exist probability measures  $\mu^\Box_\CW(\theta)  \in \cptprob \ (\theta  \in \Theta_\CW(p,d), p,d \in \N)$ which satisfy the  following conditions.
	\begin{enumerate}
		
		\item     Let $M>0$, $\gamma>0$, $(p_d )_{d \in \N}$, $p_d \geq d$,  $\sup_{d \in \N}(p_d/d) < \infty$, and $\phi_d \in \Theta_\CW(p_d, d, M)$ $(d \in \N)$.
		Then $\mbb{P}$-almost surely
		\[
		\lim_{d \to \infty}\sup_{ \theta \in \Theta_\CW(p_d, M)} \abs{%
			H_\gamma \left[ \ESD\left(\rvW_\CW(\phi_d)\right), \mu^\Box_\CW(\theta) \right] %
			- H_\gamma \left[\mu^\Box_\CW(\phi_d), \mu^\Box_\CW(\theta)\right] %
		} = 0.
		\]
		\item              Let us define the maps $\mcR \colon \HM{} \times \R^p \to \C$ and $\mcG \colon \HM{} \times \HP{} \times \R^p \to \HP{}$ by
		\begin{align}
		\mcR(b,v) := \frac{1}{d}\sum_{i=1}^{p}\frac{ v_i}{ 1 - v_i b},\
		\mcG(b, z, v) :=  [z - \mcR(b, v)]^{-1}.
		\end{align}
		Let $z \in \HP{}$ and  $v  \in \R^p$ with $\Im z > \max \{ \abs{v_k} \mid k = 1, \dots, p\}$. Then  for any initial point $b_0 \in \HP{}$, 
		we have
		\begin{align}
		\cc{\mu_\CW^\Box(d, \iota_p(v))}(z) = \lim_{ n \to \infty } \mcG_{z,v}^n(b_0),
		\end{align}         
		where $\mcG_{z, v}(b_0) := \mcG(b_0,z, v)$
		\item $\mu_\CW^\Box(d,A) = \mu_\CW^\Box( d, U^*AU)$ for any $(d, A) \in \Theta_\CW(p,d)$ and any unitary matrix $U \in M_p(\C)$.
		\item  The function $ G_{\mu_\CW^\Box(\cdot)}(z) \colon \Theta_\CW(p,d)  \to  \C$ is of class C$^\infty$ for any $z \in \HP$, and $p,d \in \N$.
		In addition, $ \mr{id} - D_{ G(z,v)}\mcG_{z,v}$ is invertible, where $D_{G(z,v)} \mcG_{z,v}$ is the derivation, for any $z \in \HP$ and $v \in \R^p$. 
	\end{enumerate}
	
\end{thm}

\begin{proof}
	The proof is postponed to Section~\ref{sscec:proofs-of-main}.
\end{proof}

\begin{thm}\label{thm:main_spn}    
	There exist deterministic probability  measures $\mu_\SPN^\Box(\theta)$  $(\theta \in \Theta_\SPN(p,d), p, d \in \N)$ which satisfy the following conditions.
	\begin{enumerate}
		\item          
		For any $M>0$, $\gamma>0$, and $(p_d )_{d \in \N}$ such that $p_d \in \N $ with $p_d \geq d$ and  $\sup_{d \in \N}(p_d/d) < \infty$, it holds that
		$\mbb{P}$-almost surely
		\[
		\lim_{d \to \infty}\sup_{( \theta) \in \Theta_\SPN(p_d,d,M)} \abs{ H_\gamma\left[ \ESD \left(\rvW_\SPN(\phi_d)\right), \mu_\SPN^\Box( \theta) \right] %
			- H_\gamma \left[\mu_\SPN^\Box(\phi_d), \mu_\SPN^\Box(\theta)\right] }  = 0.
		\]
		\item          Let $(A,\sigma) \in \Theta_\SPN(p,d)$ and $a \in \R^d$ be the vector of the singular values of $A$.
		Let us define $\eta_2 \colon \C^2 \to \C^2$ and  $\mcG_{Z, \sigma}\colon \HM(\C^2) \ \times  \R \to \HM(\C^2)$ for $Z \in \HP(\C^2)$ by 
		\[
		\eta(x,y) = ((p/d)y, x), \ \mcG_{Z,\sigma}(B):= ( Z - \sigma^2 \eta_2(B) )^{-1}.
		\]         
		Then the following operator norm limit exists;
		\begin{align}
		\ccop{\sigma}{\C^2}(Z) := \lim_{ n\to \infty} \mcG_{Z,\sigma}(G_0),
		\end{align}
		where the limit does not depend on the choice of the initial point $G_0 \in \HM(\C^2)$.
		In addition, for $a \in \R^d$, let us define 
		\begin{align}
		\ccop{a}{\C^2} \left( b_1, b_2 \right) := (\frac{b_2}{d}\sum_{k=1}^d \frac{1}{b_2b_1 - a_k^2}, %
		\frac{b_1}{p} \sum_{k=1}^{d}\frac{1}{b_2b_1 - a_k^2} + \frac{p-d}{pb_2}),
		\end{align}
		and $h_a(B) := G_a^{\C^2}(B)^{-1} - B$, $h_\sigma(B) := G_\sigma^{\C^2}(B)^{-1} - B$, where  $B =(b_1,b_2) \in \HP(\C^2)$. 
		Moreover let us define a map $\Psi \colon \HP(\C^2) \times \HP(\C^2) \times \R^d \times \R \to \HP(\C^2) \to \HP(\C^2)$ by
		\begin{align}
		\Psi(B, Z, a, \sigma ) := h_a( h_\sigma (B) + Z ) + Z,
		\end{align}
		and we write $\Psi_{Z, a, \sigma}(B) := \Psi(B, Z, a, \sigma )$.
		Then the following limit in the operator norm topology exists;
		\begin{align}
		\psi(Z, a, \sigma) := \lim_{n \to \infty} \Psi_{Z, a,  \sigma}^n(\psi_0),
		\end{align}
		and it does not depend on the choice of the initial point $\psi_0 \in \HP(\C^2)$. 
		Lastly, the  Cauchy transform $\cc{\mu_\SPN^\Box(A, \sigma)}(z)$ $( z \in \HP{})$ is equal to the first entry of a vector in $\C^2$;
		\begin{align}
		\cc{\mu_\SPN^\Box(A, \sigma)}(z) =     \cc{\mu_\SPN^\Box(\iota^d_{p,d}(a), \sigma)}(z) = \frac{1}{\sqrt{z}} \ccop{\sigma}{\C^2}\left(\psi (\sqrt{z}I_2, a, \sigma ) \right)_1.
		\end{align}

		\item $\mu_\SPN^\Box( A, \sigma) = \mu_\SPN^\Box( UAV, \sigma)$ for any $(A, \sigma) \in \Theta_\SPN(p,d)$, and any pair of unitary matrices $U \in M_p(\C)$, $V \in M_d(\C)$.
		
		\item The function $G_{\mu_\SPN^\Box (\cdot)}(z) \colon  \Theta_\SPN(p_d,d) \to \C$ is of class C$^\infty$ for any $z \in \HP$, and  $p,d \in \N$.
		In addition, $\mr{id} - D_{G_\sigma^{\C^2}(Z)} \mcG_{Z,\sigma}$ and  $\mr{id} - D_{ \psi(Z, a, \sigma)}\Psi_{Z, a, \sigma}$  are invertible for any $Z \in \HP(\C^2)$,  $a \in \R^d$ and $\sigma \in \R$, where $D_B \mcG_{Z,\sigma}$ and $D_B\Psi_{Z, a, \sigma}$ are derivations at $B \in \HP(\C^2)$.
	\end{enumerate}

\end{thm}
\begin{proof}
	The proof is postponed to Section~\ref{sscec:proofs-of-main}.
\end{proof}

\begin{rmk}
	The partial derivations $\grad_\theta  G_{\mu^\Box_\cdot (\theta)}$ $(\cdot=\CW, \SPN)$ are computed by the derivations of implicit functions $\mcG$, $\Psi$   and the chain rule,  since $\mr{id} - D\mcG$ and $\mr{id} - D\Psi$ are invertible.
	For SPN model, see Corollary~\ref{cor:derivation_omega} of the chain rule.
	Now, 
	\begin{align}\label{align:grad_of_loss}
	\grad_\theta \ell_\gamma(x,  \mu_\cdot^\Box(\theta)) = - \frac{\Im \grad_\theta  G_{\mu_\cdot^\Box(\theta)}(x+i\gamma)}{\Im G_{\mu_\cdot^\Box(\theta)}(x+i\gamma)}.
	\end{align}
\end{rmk}

\subsection{Optimization}\label{ssec:optimization}
\hfill

\noindent Our optimization algorithm is a modification of an online gradient descent (OGD, for short), which is a stochastic approximation of gradient descent. 
\begin{rmk}
	If a naive OGD is applied to a convex objective function $J(\theta)$, then $J(\theta_N) - \min_\theta J(\theta)$, where $\theta_N$ is the updated parameter, becomes $O(1/\sqrt{N})$ after the $N$ iteration \cite{nemirovski2009robust}.
	If the OGD is applied to a non-convex smooth objective function, then its gradient $\nabla_\theta J(\theta_N)$ converges to $0$ almost surely as $N \to \infty$ under some additional assumptions 
	(see \cite{bottou1998online} for more detail).
	An origin of OGD is the stochastic approximation by Robbins-Monro \cite{robbins1951stochastic}.
\end{rmk}

Algorithm~\ref{alg:BOGD-FDE} shows the OGD-based optimization algorithm using the gradient of Cauchy noise loss.
Fix a self-adjoint random matrix model $W(\theta)$ and write $\ell_\gamma(x, \theta) = \ell_\gamma(x, \mu^\Box(\theta))$.
The algorithm requires settings of the maximum number $N$ of iterations, an update rule of the parameters $\gamma$ of the Cauchy noise loss,  the initial parameter $\theta_0$ of the model, and the bounded convex parameter space $\Theta \subset \R^k$ for a $k \in \N$.

\begin{algorithm}[htbp]
	\caption{Online Gradient Decent Optimization of  FDE model}
	\label{alg:BOGD-FDE}
	\begin{algorithmic}[1]
		\Require A $d \times d$ self-adjoint matrix $W$ 
		\State $\lambda = (\lambda_1 ,\dots, \lambda_d) \in \R^d  \gets  $ eigenvalues of $W$
		\State Initialize $\theta_0$
		\While{ $  0 \leq n < N$ }
		\State Choose an index $j$ uniformly from $\{1, \dots, d\}$.
		\State Generate a Cauchy noise $T$ of scale $\gamma$.
		\State $x  \gets \lambda_j + T$ \Comment{Add  a Cauchy noise}
		\State Compute $G_{\mu^\Box(\theta_n)}(x + i\gamma)$ by iterative methods. 
		\State Compute  $\grad{}_\theta |_{\theta=\theta_{n}} G_{\mu^\Box(\theta)} (x + i \gamma)$ by the chain rule. 
		\State Calculate the gradient  $\grad_\theta |_{\theta=\theta_{n}} \ell_\gamma(x, \theta)$ of the Cauchy noise loss by \eqref{align:grad_of_loss}.
		\State Update $\theta_{n + 1}$ based on $\theta_n$ and $\grad_\theta |_{\theta=\theta_n} \ell_\gamma(x, \theta)$ according to the update rule \eqref{align:adam}.
		\State $\theta_{n+1} \gets \Pi (\theta_{n + 1})$    \Comment{Project onto  $\Theta$.}
		\State $n  \gets n+1$
		\EndWhile
		\Ensure{$\theta_N$}
	\end{algorithmic}
\end{algorithm}

In our method, the sample is assumed to be a single-shot observation of a self-adjoint square random matrix. Through the algorithm, the scale $\gamma >0$ is fixed.
We consider the collection $\{\lambda_1, \dots, \lambda_d\}$ of the eigenvalues of the sample matrix.
Each iteration of the algorithm consists of the following steps. We continue the iteration while  $n < N$.

First, we generate an index  $j$ from the uniform distribution on $\{ 1, 2, \dots, d\}$, and generate a Cauchy noise $T$ of scale $\gamma$.
We generate them independently throughout all iterations.

Second,   we compute $\grad_\theta \ell_\gamma(\lambda_j + T + i\gamma, \theta)$  by  \eqref{align:grad_of_loss}.

Third, we update parameters by using the gradient of the loss function. 
We use Adam \cite{kingma2014adam} since it requires little tuning of hyperparameters. It is defined as follows.
Assume that $\theta \in \R^k$.
Let $m_0 = v_0 = 0 \in \R^k$ and for $n = 0, 1, \dots, N-1$,  Adam uses  the following recurrence formula;
\begin{align}
m_{n+1}& = \beta_1 m_{n} + ( 1 - \beta_1)  \grad_\theta\ell_\gamma(x, \theta)|_{\theta=\theta_n},\\
v_{n+1}&= \beta_2 v_{n}  + ( 1- \beta_2)  (\grad_\theta\ell_\gamma(x, \theta)|_{\theta=\theta_n})^2,\\
\hat{m}_{n+1} &=  m_{n+1} / ( 1 - \beta^{n+1}_1),\\
\hat{v}_{n+1} & =  v_{n+1} / ( 1 - \beta_2^{n+1}),\\
\theta_{n+1} &= \theta_n - \alpha ( \sqrt{\hat{v}_{n+1}} + \eps )^{-1} \hat{m}_{n+1} ,\label{align:adam}
\end{align}
where the product  and  the division of vectors are entrywise.
In addition,  $\alpha, \beta_1, \beta_2 > 0$ are constants  such that $\beta_1, \beta_2 <  1$, which control the exponential moving average,  and $\eps > 0, \eps \approx 0$ is a small value  for preventing division by zero.
Adam adaptively estimates the first and second moments of gradients.
Note that our loss function is non-convex and the convergence of Adam is proven for convex loss functions \cite{kingma2014adam}.  

Lastly, we project parameters onto a convex parameter space $\Theta$.

\begin{rmk}
	Adam is a diagonal method based on the empirical Fisher matrix (see \cite{kingma2014adam} and \cite{martens2014new} for the detail).
	The vector $\hat{v}_{n}$ approximates the diagonal part of the empirical version of the Fisher information matrix.
	Now, the natural gradient descent \cite{amari1998natural}  is based on the information geometry, which updates parameters to the direction of the steepest direction in the Fisher information metric given by the Fisher information matrix.
	
	Note that our loss function is  given by 
	the log of the new density $- {\pi}^{-1}\Im G_{\mu^\Box_\theta}( \cdot + i\gamma)$ instead of the log of the traditional likelihood  \eqref{align:likelihood-origin}.
\end{rmk}

\section{Theory}

\subsection{Free Deterministic Equivalents}\label{ssec:fde}\hfill

\noindent In this section, we reformulate free deterministic equivalents introduced by  Speicher-Vargas \cite{speicher2012free}.
To run our algorithm, we do not require infinite-dimensional operators, but only require using finite-dimensional ones.
However, infinite-dimensional operators have theoretically critical roles.

First, we summarize some definitions from operator algebras and free probability theory. See \cite{Mingo2017free} for the detail.
\begin{defn}\hfill
	\begin{enumerate}
		\item A \emph{C$^*$-probability space}  is a pair $(\mf{A}, \tau)$  satisfying followings.
		\begin{enumerate}
			\item The set $\mf{A}$ is \emph{a unital $C^*$-algebra}, that is, a possibly non-commutative subalgebra of the algebra  $B(\hilb)$ of bounded $\C$-linear operators on a Hilbert space $\hilb$ over $\C$ satisfying the following conditions:
			\begin{enumerate}
				\item it is stable under the adjoint $* : a \to a^*, a \in \mf{A}$,
				\item it is closed under  the topology of the operator norm of $B(\hilb)$,
				\item it contains the identity operator $\id_\hilb$  as the unit $1_\mf{A}$ of $\mf{A}$.
			\end{enumerate}
			\item The function $\tau$ on  $\mf{A}$ is a \emph{faithful tracial state}, that, is  a $\C$-valued  linear functional with
			\begin{enumerate}
				\item $\tau(a)\geq 0$ for any $a \geq 0$, and the equality holds if and only if $a=0$,
				\item $\tau(1_\mf{A})=1$,
				\item $\tau(ab)=\tau(ba)$ for  any $a, b \in \mf{A}$.
			\end{enumerate}
		\end{enumerate}
		\item A possibly non-commutative subalgebra $\mf{B}$ of a C$^*$-algebra $\mf{A}$ is called a \emph{$*$-subalgebra} if $\mf{B}$ is stable under the adjoint operator $*$. Moreover, it is called a \emph{unital C$^*$-subalgebra} if the $*$-subalgebra is closed under the operator norm topology and contains $1_\mf{A}$ as its unit.
		\item Two unital $C^*$-algebras are called \emph{$*$-isomorphic} if there is a bijective linear map between them which preserves the $*$-operation and the multiplication.
		\item Let us denote by $\mf{A}_\sa$ the set of \emph{self-adjoint} elements, that is, $a = a^*$ of $\mf{A}$.
		\item Write $\Real a := (a + a^*)/2$ and $\Image a := ( a- a^*)/{2i}$ for any $a \in \mf{A}$.
		\item The \emph{distribution} of  $a \in \mf{A}_\sa$  is the probability measure $\mu_a \in \cptprob$ determined by
		\[
		\int x^k \mu_a(dx) = \tau(a^k),\  k \in \N.
		\]
		\item For $a \in \mf{A}_\sa$, we define its Cauchy transform $G_a$ by $G_a(z):= \tau[ (z -a )^{-1} ] \ (z \in \C \setminus \R)$, equivalently, $G_a:=G_{\mu_a}$.
		
	\end{enumerate}
	
\end{defn}

\begin{defn}
	A family of $*$-subalgebras $(\mf{A}_j)_{j \in J}$ of $\mf{A}$ is said to be \emph{free}
	if the following factorization rule holds: for any $n \in \N$ and indexes $j_1, j_2, \dots, j_n \in J $ with $j_1 \neq j_2 \neq j_3  \neq \cdots \neq j_n$, and $a_l \in \mf{A}_l$ with $\tau(a_l)=0$ $(l = 1 ,\dots, n)$, it holds that
	\[
	\tau(a_1 \cdots a_l) = 0.
	\]
	Let  $(x_j)_{j \in J}$ be a family of self-adjoint elements $x_j \in \mf{A}_\sa$. For $j \in J$, let $\mf{A}_j$ be the $*$-subalgebra of polynomials of $x_j$.
	Then $(x_j)_{j \in J}$ is said to be  free if $\mf{A}_j$  is free.
\end{defn}

We introduce special elements in a non-commutative probability space.
\begin{defn}
	Let $(\mf{A}, \tau)$  be a C$^*$-probability space.
	\begin{enumerate}
		\item
		An element $s \in \mf{A}_\sa$ is called \emph{standard semicircular} if its distribution is given by the standard semicircular law;
		\[
		\mu_s(dx) = \frac{\sqrt{4-x^2}}{2\pi}{\bf 1}_{[-2,2]}(x)dx,
		\]
		where ${\bf 1}_S$ is the indicator function for any subset $S \subset \R$.
		\item Let $v > 0$.    An element $ c \in \mf{A}$ is called \emph{circular of variance  $v$} if
		\[c=\sqrt{v}\frac{s_1+is_2}{\sqrt{2}},\]
		where $(s_1,s_2) $ is a pair of free standard semicircular elements.
		\item A \emph{$*$-free circular family} (resp.\,standard $*$-free circular family) is a family $\{ c_j \mid j \in J \}$ of circular elements $c_j \in \mf{A}$ such that $\bigcup_{j\in J}\{ \Real c_j , \Image c_j \}$ is free (resp. and each elements is of variance $1$).
	\end{enumerate}

\end{defn}

A free deterministic equivalent (FDE, for short) of a PGM model is constructed by replacing Ginibre matrix $Z$ by matrices of circular elements.
\begin{rmk}
	Equivalently, FDE is obtained by taking the limit of amplified models which is constructed by (1) copying deterministic matrices by taking a tensor product with identity and (2) each $Z \in \mr{GM}(\cdot, \cdot, \K)$ is enlarged by simply increasing the number of entries.
	See \cite{speicher2012free} and \cite[pp.19]{calros2015free} for the detail.
	
	Note that the original definition of FDE treats not only Ginibre matrices but also more general random matrices.
	
\end{rmk}

We reformulate FDE for random matrix models.
\begin{defn}[Free Deterministic Equivalents]\label{defn:FDE}
	Fix a  C$^*$-probability space $(\mf{A}, \tau)$.
	Let $\bar{P}_\mf{J} : P_\mf{J}^\Box \colon \prod_{k=1 }^n M_{p_k, d_k}(\C) \to M_{p,d}(\RVallmoments)$ be a Ginibre matrix model over $R$ or $\C$ of type $(P,\mf{I})$ on a subset $\Theta \subset \prod_{k=1 }^n M_{p_k, d_k}(\C)$.
	Then its \emph{free deterministic equivalent} (FDE for short)  is
	the map $P_\mf{J}^\Box \colon \Theta \to M_{p,d}(\mf{A})$
	defined as
	\[
	P_\mf{J}^\Box(D_1, \dots, D_n) := P(C_1, \dots, C_m , D_1, \dots, D_n),
	\]
	where $C_1, \dots, C_m $ are rectangular matrices such that the collection of all rescaled entries
	\[
	\{ \sqrt{\ell_k}C_k(i,j)\mid i = 1, \dots, r_k, j= 1, \dots, \ell_k, k =  1, \dots, m \}
	\]
	is a standard $*$-free circular family in $(\mf{A}, \tau)$.

	In addition, if $p=d$ and FDE is self-adjoint for any elements in $\Theta$, the Cauchy transform of FDE is called the \emph{deterministic equivalent}.
\end{defn}

Now, each coefficient $\sqrt{\ell_k}$ is multiplied for the compatibility with the normalization of Gaussian random matrices.
Besides, we note that the FDE model does not depend on the field $\R$ or $\C$.

\begin{defn}\label{defn_fde_model}
	Let $p,d \in \N$.
	\begin{enumerate}
		\item The \emph{free deterministic equivalent signal-plus-noise model (FDESPN model, for short) of type $(p,d)$} is defined as the FDE $W^\Box_\SPN \colon M_{p, d}(\C) \times \C \to M_d(\mf{A})_+$ of $\rvW_\SPN$.
		Note that
		\[
		W^\Box_\SPN(A, \sigma) = (A + \sigma C)^*(A+ \sigma C),
		\]
		where $(C_{ij}/\sqrt{d})_{i=1, \dots, p, j= 1, \dots, d}$ is a standard $*$-free circular family.
		\item The \emph{free deterministic equivalent compound Wishart model (FDECW model, for short) of type $(p,d)$} is defined as the FDE $W_\CW^\Box \colon M_p(\C) \to M_d(\mf{A})$ of $\rvW_\CW$.
		Note that
		\[
		W_\CW^\Box (A) = C^*AC,
		\]
		where $(C_{ij}/\sqrt{d})_{i=1, \dots, p, j= 1, \dots, d}$ is a standard $*$-free circular family.
		
	\end{enumerate}
	
\end{defn}

\begin{rmk}
	
	By Remark~\ref{rmk:reduce_model}, we consider the  optimization of restricted models on $\Theta_\CW(p, M)$ or $\Theta_\SPN(p,d,M)$ for an $M < \infty$.
	The boundedness of each parameter space is required for the uniform convergence of our loss function (see  Corollary~\ref{cor_cnl_converges}).

\end{rmk}

We reformulate the convergence of the gap between FDE and the original random matrix model in the case parameters belong to bounded sets.
\begin{prop} \label{prop_as_convergence_of_moments}
	\hfill
	\begin{enumerate}
		\item Let us consider following sequences;
		\begin{enumerate}
			\item $(p_d )_{d \in \N}$ such that $p_d \in \N $ with $p_d \geq d$ and  $\sup_{d \in \N}(p_d/d) < \infty$,
			\item $(A_d)_{d \in \N}$ such that  $A_d \in M_{p_d}(\C)$ and $\sup_{d \in\N}\norm{A_d } < \infty$.
		\end{enumerate}
		Then for any $k\in\N$, it holds that          $\mbb{P}$-almost surely,
		\begin{align*}
		m_k ( \rvW_\CW(A_d) ) - m_k(W_\CW^\Box(A_d))  \to 0 \text{, as $d \to \infty$}.
		\end{align*}
		\item Let us consider following sequences;
		\begin{enumerate}
			\item $(p_d )_{d \in \N}$ such that $p_d \in \N $ with $p_d \geq d$ and  $\sup_{d \in \N}(p_d/d) < \infty$,
			\item $(A_d)_{d \in \N}$ such that  $A_d \in M_{p_d, d}(\C)$ and $\sup_{d \in \N}\norm{A_d^* A_d } < \infty$,
			\item $(\sigma_d)_{d \in \N} $ such that $\sigma_d \in \R$ and $ \sup_{d \in \N}\abs {\sigma_d } < \infty$.
		\end{enumerate}
		Then for any $k\in\N$, it holds that          $\mbb{P}$-almost surely,
		\begin{align*}
		m_k ( \rvW_\SPN(A_d, \sigma_d) - m_k(W_\SPN^\Box(A_d, \sigma_d) ) \to 0 \text{, as $d \to \infty$}.
		\end{align*}
		
	\end{enumerate}
\end{prop}
\begin{proof}
	This proposition is well known. For the reader's convenience, we prove it.
	Firstly let us consider the case of CW model.
	By the genus expansion of Ginibre matrices (see \cite{Mingo2017free} for both real and complex case)
	and  the uniform boundedness of the sequences, there are  constants $K_{1,k} ( k \in \N)$ such that  for any $k,d \in \N$,
	\begin{align}
	\abs { \mbb{E}[m_k ( \rvW_\CW(A_d)] - m_k(W_\CW^\Box(A_d))} \leq \frac{K_{1,k}}{d}.
	\end{align}
	By the expansion of the fluctuation of Ginibre matrices (see \cite{collins2007second} for the complex case, \cite{redelmeier2014real} for the real case) and  the uniform boundedness of the sequences,
	there are constants  $K_{2,k} > 0 \ ( k \in \N)$  such that for any $k,d \in \N$,
	\begin{align}
	\mbb{V}\left[    m_k( \rvW_\CW(A_d) )\right] \leq \frac{K_{2,k}}{d^2}.
	\end{align}
	As the consequence, we have for any $k,d \in \N$,
	\begin{align}
	\mbb{E} \left[  \abs{   m_k( \rvW_\CW(A_d) ) -   m_k(W_\CW^\Box(A_d)) }^2 \right] \leq \frac{K_{2,k} + K_{1,k}^2}{d^2}.
	\end{align}
	Since $\sum_{d=1}^{\infty}(K_{2,k} + K_{1,k}^2)/d^2 < \infty$, the almost-sure convergence holds.
	
	The proof for  SPN model is given by the same argument.
\end{proof}

\subsection{Analysis of Determination Gap}\label{ssec:fde-geeralization-gap}
\hfill

\noindent In this section we prove the \emph{determination gap}, which is  defined as the following, converges to $0$ $\mbb{P}$-almost surely and uniformly on a bounded parameter space as the matrix size becomes large. 
Note that the determination gap has the same role as the generalization gap in the empirical cross-entropy method.

\begin{defn}(Determination Gap)
	Let $(W_\theta)_{\theta \in \Theta}$  be a self-adjoint PGM model and  $(W^\Box_\theta)_{\theta \in \Theta}$ be its FDE. Then the \emph{deterministic gap} is the real random variable defined as 
	\begin{align}
	H_\gamma(\ESD(W_{\theta_0}), \mu_{W_{\theta}^\Box} )
	- H_\gamma(\mu_{W_{\theta_0}^\Box}, \mu_{W_{\theta}^\Box} ),
	\end{align}
	where $\theta, \theta_0 \in \Theta$.
\end{defn}

\bigskip
\bpara{Basic Properties}\label{sssec:basic-properties}

Let us recall on the entropy for strictly positive continuous probability density functions.

\begin{defn}\hfill
	
	\begin{enumerate}
		\item Let us denote by $\PDF$ the set  $\{p \in C(\R) \mid p > 0, \int p(x)dx =  1 \}$.
		\item For any $q \in \PDF$, we denote by $L^1(q):= \{ f \colon \R \to \C \mid \text{Borel  measurable  and  $ fq \in L^1(\R)$}\}$.
		\item Let  $p,q \in \PDF$ with $ \log p \in L^1(q)$.
		The cross-entropy of $p$ against the reference $q$ is defined by
		\[
		H[q , p ] := - \int q(x) \log p(x) dx.
		\]
		\item For any $q \in \PDF$ such that $\log q \in L^1(q)$, we denote by $S(q) := H[q,q]$ the entropy of $q$.
		\item
		The relative entropy (or  Kullback-Leibler divergence) of $p \in \PDF$ against to $q \in \PDF$ is defined as
		$D_\mr{KL}[q \| p] :=  H[q, p] - S(q)$.
		
	\end{enumerate}
\end{defn}
We reformulate \emph{the principle of minimum cross-entropy} for $\PDF$ as the following.
\begin{prop}\label{prop_PMC}
	Let $q \in \PDF$ with $\log q \in L^1(q)$. Then
	\[
	\argmin_{p \in \PDF, \ \log p \in L^1(q) } D_\mr{KL}[q \|  p] = \argmin_{p \in \PDF, \ \log p \in L^1(q) } H[q,  p] = \{q\},
	\]
	where $\arg\min_{x \in X} f(x) := f^{-1}(\min_{x \in X}f(x))$  for a $\R$-valued function $f$ on a set $X$ which has a minimum point.
\end{prop}
\begin{proof}This is well-known.
	Let $p \in \PDF $ with $\log p \in L^1(q)$. Then $-q\log (p/q) \geq -q (p/q - 1 ) =  q - p$.
	Hence $D(q \| p) \geq 0$ and the equality holds if and only if $\log(p/q) = p/q -1$, which proves the assertion.
\end{proof}

We begin by proving the Cauchy cross-entropy is well-defined.

\begin{lem}\label{lemma:cne-well-defined}
	For any Cauchy random variable $\rvT \sim \mr{Cauchy}(0,\gamma)$ with $\gamma >0$, the random variable $\log P_\gamma(\rvT)$ has every absolute moments;
	\begin{align}\label{align:moments_of_log}
	\mbb{E} \left[  \abs{\log P_\gamma(\rvT)}^k\right] = \int \abs{\log P_\gamma(t) }^k P_\gamma(t) dt < \infty.
	\end{align}
	In particular, the entropy of each Cauchy random variable is well-defined.
\end{lem}
\begin{proof}
	This lemma is well-known; this follows from  the facts $\int_0^{\pi/2}  1/\sqrt{\phi} d\phi < \infty$ and
	$\phi^{1/2k} \abs{\log(\sin \phi) }$ %
	$\leq \phi^{1/2k} (\abs{\log( \sin\phi/\phi)} +\abs{\log\phi})$%
	$\to 0  \text{ as $\phi \to +0$}$,
	for any  $k \in \N$.    
\end{proof}

\begin{lem}
	Write  $\mc{B}_M := \{ \mu \in \cptprob \mid  \supp \mu \subset [-M, M] \}$ for $M>0$.    
	Fix $\gamma >0$. Then the following hold.
	\begin{enumerate}
		\item
		For any $\mu \in \mc{B}_M$, it holds that
		$        (2 + 2M^2/\gamma^2)^{-1}P_\gamma \leq P_\gamma*\mu \leq (\pi \gamma)^{-1}.$
		\item $\sup_{\mu \in \mc{B}_M}\abs{\log (P_\gamma*\mu(t)) } \leq \abs{\log P_\gamma(t)} + C_{M,\gamma}$ for any $t \in \R$, where $C_{M,\gamma} = \log  (2 +  2M^2/\gamma^2 ) + \abs{\log \pi \gamma}$ .
		\item $\log(P_\gamma*\mu) \in L^1(P_\gamma*\nu)$ for any $\mu \in \cptprob$ and $\nu \in \mc{B}(\R)$ with $m_2(\nu) < \infty$. Hence the Cauchy cross-entropy $H_\gamma(\nu,\mu)$ is well-defined.
	\end{enumerate}
\end{lem}
\begin{proof}
	The inequality $P_\gamma*\mu \leq (\pi \gamma)^{-1}$ follows from the definition.
	Since $(t-s)^2 \leq 2(t^2+ M^2) $ for all $t\in \R $ and $s \in \supp(\mu) \subset [-M,M]$, we have 
	\begin{align}\label{align:bound_of_log_cauchy}
	\frac{P_\gamma(t)}{P_\gamma*\mu(t)} \leq \frac{2(t^2+M^2) + \gamma^2}{t^2+\gamma^2} = 2 + \frac{2M^2 - \gamma^2}{t^2+\gamma^2} \leq 2 +\frac{2M^2}{\gamma^2}.
	\end{align}
	Hence we have (1).
	
	In addition, $\abs{\log (P_\gamma*\mu(t)) } \leq \max \{ \abs{\log P_\gamma(t)/(2 +  2M^2/\gamma^2 )} , \abs{\log \pi \gamma} \}$, which proves (2).
	
	Let us prove (3). Let $M>0$ satisfy $\supp \mu \subset [-M,M]$.
	By Tonelli's theorem  for nonnegative measurable functions and (1), we have
	\[
	\int \abs{\log (P_\gamma*\mu(t))} P_\gamma*\nu(t)(dt) \leq  \iint  \abs{\log P_\gamma(t)}P_\gamma(t-x)dt \nu(dx) + C_{M,\gamma}.
	\]
\end{proof}

We prove the principle of minimum Cauchy cross-entropy.

\begin{proof}[proof of Proposition~\ref{prop:principal-minimum-cc}]    
	If  $H(P_\gamma*\nu , P_\gamma*\mu)$ attains the minimum, then by Proposition~\ref{prop_PMC}, we have $P_\gamma*\mu = P_\gamma*\nu$. By Lemma~\ref{lem:key_lemma_cauchy}, the assertion holds.
\end{proof}

\begin{note}
	Write $h_\gamma^\mu(x):= H_\gamma(\delta_x, \mu )$.
\end{note}

Note that $h_\gamma^\mu(x) = -\int \log (P_\gamma*\mu(t))P_\gamma(x-t)dt=-\int \log (P_\gamma*\mu(x-t)) P_\gamma(t)dt$.
\begin{lem}\label{lemma_shifted_entropy}
	Fix $\gamma > 0$. Then $\log(P_\gamma*\mu)$ and $h_\gamma^\mu$ are differentiable  and the followings hold.
	\begin{enumerate}
		\item $\sup_{ \mu \in \borel } \norm{ \left[\log(P_\gamma*\mu)\right]^\prime }_\infty \leq 1/\gamma$.
		\item  $\sup_{\mu \in \borel} \norm{ (h_\gamma^\mu )^\prime}_\infty \leq 1/\gamma$.
		\item For any  $L,M >0$ and $\nu \in \mc{B(\R)}$,
		\[
		\sup_{ \mu \in \mc{B}_M} \int_{ \abs{x} \geq L} \abs{h_\gamma^\mu(x)} \nu(dx) \leq  \frac{1}{L} \left[ \left(\abs{S(P_\gamma)} + C_{M,\gamma} \right)m_2(\nu)^{1/2} + \frac{1}{\gamma} m_2(\nu)\right].
		\]
	\end{enumerate}
	
\end{lem}

\begin{proof}
	For any $s \in \R$, by the Cauchy-Schwarz inequality, it holds that
	\begin{align}
	P_\gamma^\prime(s) = \frac{1}{\gamma}\frac{2\gamma s}{s^2 + \gamma^2}\frac{\gamma}{\pi(s^2 + \gamma^2 )}  \leq \frac{1}{\gamma} P_\gamma(s),
	\end{align}
	and implies that $\abs{     P_\gamma^\prime(s)} \leq \gamma^{-1}P_\gamma(s)$. 
	In particular, $P_\gamma^\prime$ is bounded.
	Thus  $(P_\gamma*\mu)^\prime(y) = \int P_\gamma^\prime(y-s)\mu(ds)$, which implies (1).

	In particular, $\abs {(P_\gamma*\mu)^\prime/(P_\gamma*\mu)} P_\gamma  \leq \gamma^{-1} P_\gamma \in L^1(\R)$.
	Therefore, $h_\gamma^\mu$ is differentiable and for any $x \in \R$,
	\begin{align}
	\abs{(h_\gamma^\mu)^\prime (x) } = \abs{ \int \frac{(P_\gamma*\mu)^\prime(x-t)}{P_\gamma*\mu(x-t)}P_\gamma(t)dt }\leq \frac{1}{\gamma},
	\end{align}
	which proves (2).
	
	By (2), $\abs{h_\gamma^\mu(x) } \leq \abs{h_\gamma^\mu(0) }+ \gamma^{-1}\abs{x} \leq \left(\abs{S(P_\gamma)} + C_{M,\gamma}\right) +\gamma^{-1}\abs{x}$.
	Hence
	\[\int \abs{x h_\gamma^\mu(x) } \nu(dx) \leq \left(\abs{S(P_\gamma)} + C_{M,\gamma}\right) \int \abs{x} \nu(dx)  +\gamma^{-1}m_2(\nu),
	\]
	which proves (3).
	
\end{proof}

\bigskip
\bpara{Determination Gap}\label{sssec:determination-gap}
Here we show an asymptotic property of the determination gap.
It is known that the convergence in moments of a sequence of compact support probability measures implies its weak convergence (see \cite[Theorem~30.2]{billingsley2008probability}).
First, we prove its stronger version.

\begin{lem}\label{lemma_moment_to_weak_in_cauchy}
	We denote by $\R[X]$ the set of  $\R$-coefficient polynomials, and write 
	$C_b(\R) := \{ f \in C(\R) \mid \sup_{ x \in \R } \abs{f(x)} < \infty \}$.
	Consider  arbitrary sequence $ (\nu_{d}, \nu^\Box_d)_{d \in \N}$  of pairs  with $\nu_d \in \allmoments, \nu^\Box_d \in \cptprob$ satisfying
	\begin{enumerate}
		\item $\lim_{d \to \infty}\abs{m_k(\nu_{d}) - m_k(\nu^\Box_d) }=0 \ (k \in \N)$,
		\item there is $ M> 0$ such that $\nu^\Box_d \in \mc{B}_M \ ( d \in \N)$.
	\end{enumerate}
	Then for any $f \in C_b(\R)$ and $p \in \R[X]$,  we have
	\[    \lim_{d \to \infty}\abs{\int f(x)p(x)\nu_d(dx) - \int f(x)p(x)\nu^\Box_d(dx)} = 0.\]
\end{lem}

\begin{proof}
	Let $\ell \in \N$ such that $\sup_{x \in \R} \abs{p(x)/ ( x^{2\ell} + 1)} < \infty$. Then $f(x) p(x) = g(x) (x^{2\ell} + 1)$ where $g(x)= f(x)p(x)/ ( x^{2\ell} + 1) \in C_b(\R)$.
	Hence without loss of generality, we may assume that $p(x) = x^{2\ell} + 1$. Next, for any $\mu \in \cptprob$, let us define $\mu^p \in \cptprob$ as
	\begin{align}
	\mu^p(dx) &:= \frac{1}{ \int p(y)  \mu (dy) }  p(x) \mu(dx).
	\end{align}
	Then
	\begin{align}
	\abs {m_k(\nu_d^p) - m_k({\nu^\Box_d}^p ) }  &=  \frac{\abs { \left( m_{2\ell}(\nu^\Box_d) + 1 \right)m_k(\nu_d) - \left( m_{2\ell}(\nu_d) + 1  \right)m_k(\nu^\Box_d ) }
	}{ \left( m_{2\ell}(\nu_d) + 1  \right)\left( m_{2\ell}(\nu^\Box_d) + 1 \right)} \\
	&\leq \abs {  \left( m_{2\ell}(\nu^\Box_d) + 1 \right)m_k(\nu_d) - \left( m_{2\ell}(\nu_d) + 1  \right)m_k(\nu^\Box_d ) }\\
	&\leq  \left( M^{2\ell}+ 1 \right)  \abs {m_k(\nu_d) -m_k(\nu^\Box_d) }  + \abs{ m_{2\ell}(\nu_d) -  m_{2\ell}(\nu^\Box_d)} M^k \to 0 \ ( d \to \infty).
	\end{align}
	where we used $\sup_{ d\in \N} \abs {m_k( \nu^\Box_d) } \leq M^k < \infty$ for any $k \in \N$.
	Hence  the sequence $(\nu_d^p, {\nu^\Box_d}^p )_{d\in \N}$ satisfies  the condition (1). It is clear that $({\nu^{\Box}_d}^p )_{d\in \N}$ satisfies  the condition (2). Hence without loss of generality, we only need to show the following;     for any $f \in C_b(\R)$,
	\begin{align}\label{align:claim_by_subsub}
	\lim_{d \to \infty}\abs{\int f(x)\nu_d(dx) - \int f(x)\nu^\Box_d(dx)} = 0.
	\end{align}

	To show \eqref{align:claim_by_subsub},     firstly consider arbitrary subsequence $(\nu_{d_i}^\Box)_{ i \in \N}$.
	By condition (2), the sequence $ (\nu^\Box_d)_{ d \in \N}$ is tight.
	Therefore, there exist a further subsequence $(\nu^\Box_{d_{i(j)}})_{j \in \N}$ and $\tilde{\chi} \in \mc{B}(\R)$ such that  $\nu^\Box_{d_{i(j)}}$ converges weakly to $\tilde{\chi}$ as $j \to \infty$.
	By the condition (2), it holds that $\supp \tilde{\chi} \subset [-M, M]$ and $\tilde{\chi} \in \cptprob$.
	Moreover,  by cutting off $x \mapsto x^k$ out of $[-M,M]$, the condition (2) also implies that $\nu^\Box_{d_{i(j)}}$ converges to $\tilde{\chi}$ in moments.
	By (1), $(\nu_{d_{i_j}})_j$ also converges  to $\tilde{\chi}$ in moments.
	Since $\tilde{\chi}$ has a compact support, the distribution of  $\tilde{\chi}$ is determined by its moments. Hence by \cite[Thoeorem~30.2]{billingsley2008probability}, $\nu_{d_{i_j}}$ converges weakly to $\tilde{\chi}$.
	This implies that
	\begin{align}
	\lim_{j \to \infty}\abs{\int f(x)\nu_{d_{i_j}}(dx) - \int f(x)\nu^\Box_{d_{i_j}}(dx)} = \abs{ \int f(x)\tilde{\chi}(dx) - \int f(x)\tilde{\chi}(dx)} =  0.
	\end{align}
	Hence by the sub-subsequence argument, \eqref{align:claim_by_subsub} holds.
\end{proof}

Now we are ready to state the main theorem in a general setting, which implies an asymptotic property of the determination gap.

\begin{thm}\label{thm_deterministic_convergence}
	Let $(\nu_d, \nu^\Box_d)_{d \in \N}$ satisfy the assumptions in Lemma~\ref{lemma_moment_to_weak_in_cauchy}.
	Then for any $\gamma > 0$ and $M>0$, we have
	\[
	\lim_{d \to \infty}\sup_{\mu \in \mc{B}_M}  \abs{ H_\gamma( \nu_{d} , \mu) - H_\gamma(\nu^\Box_d , \mu) }= 0.
	\]
\end{thm}

\begin{proof}
	Fix $\eps >0$.
	By Lemma~\ref{lemma_shifted_entropy}(2), there is $L>>M$ such that
	\begin{align}\sup_{d \in \N}\sup_{\mu \in \mc{B}_M} \int_{ \abs{x} \geq L} \abs{h_\gamma^\mu(x) }\nu_d(dx) ,\     \sup_{d \in \N}\sup_{\mu \in \mc{B}_M} \int_{ \abs{x} \geq L} \abs{h_\gamma^\mu(x)}\nu_d^\Box(dx)  < \eps. \label{align:uniform_bounded}
	\end{align}
	Hence such $L$ and any $d \in \N$, it holds that
	\[
	\sup_{\mu \in \mc{B}_M} \abs{H_\gamma( \nu_{d}, \mu) - H_\gamma(\nu^\Box_d, \mu) }\leq 2 \eps + \sup_{\mu \in \mc{B}_M} \abs{\int_{-L}^L h_\gamma^\mu(x)\nu_d(dx) - \int_{-L}^L h_\gamma^\mu(x)\nu_d^\Box(dx) }.
	\]
	By Lemma~\ref{lemma_shifted_entropy}(1), the family $K:= \{h_\gamma^\mu|_{[-L,L]}  \mid \mu \in \mc{B}_M\}$ is uniform bounded and equicontinuous in $C([-L,L])$.
	Hence by Ascoli-Arzela's theorem (see \cite{dudley2002real}), $K$ is totally bounded in $C([-L,L])$:  for any $\eps >0$, there are $\mu_1, \dots, \mu_n \in \mc{B}_M$ such that $K \subset \cup_{\ell=1}^n B( h^\gamma_{ \mu_\ell}|_{[-L,L]} ; \eps)$, where $B(f; \eps) := \{ g \in C([-L,L]) ; \| g - f \|_{[-L,L]} \leq \eps \}$.
	Therefore for any $\mu  \in \mc{B}_M$, it holds that
	\begin{align*}
	&\abs{\int_{-L}^L h_\gamma^\mu(x)\nu_d(dx) - \int_{-L}^L h_\gamma^\mu(x)\nu_d^\Box(dx) } \\
	&\leq \eps \nu_d([-L, L]) + \eps \nu_d^\Box([-L, L]) +
	\max_{ \ell=1,\dots, n} \abs{\int_{-L}^L h_{ \mu_\ell}(x) \nu_{d}(dx) - \int_{-L}^L h_{ \mu_\ell}(x) \nu^\Box_d(dx) }.
	\end{align*}
	By \eqref{align:uniform_bounded}, we have
	\begin{align}
	\abs{\int_{-L}^L h_{ \mu_\ell}(x) \nu_{d}(dx) - \int_{-L}^L h_{ \mu_\ell}(x) \nu^\Box_d(dx) } \leq
	2\eps +  \abs{\int h_{ \mu_\ell}(x) \nu_{d}(dx) - \int h_{ \mu_\ell}(x) \nu^\Box_d(dx) }.
	\end{align}
	Since $h_\gamma(x) = O(\abs{x})$ as $x \to \pm \infty$ by  Lemma~\ref{lemma_shifted_entropy} (1), we can apply Lemma~\ref{lemma_moment_to_weak_in_cauchy},  and there is $d_0 \in \N$ satisfying
	\[
	\sup_{d \in \N, d> d_0}\max_{ \ell = 1, \dots ,n}|\int h_{ \mu_\ell}(x) \nu_{d}(dx) - \int h_{ \mu_\ell}(x) \nu^\Box_d(dx) |  \leq \eps.
	\]
	Hence $\sup_{\mu \in \mc{B}_M} \abs{H_\gamma( \nu_{d}, \mu) - H_\gamma(\nu^\Box_d, \mu) } \leq 9 \eps$ for any $d > d_0$, which proves the assertion.
	
\end{proof}

We have the following almost-sure convergence of the determination gap.

\begin{cor}\label{cor_cnl_converges} 
	Fix $M>0$ and $\gamma>0$.
	\begin{enumerate}
		\item [(CW)]
		Under the settings (a)(b) of Proposition~\ref{prop_as_convergence_of_moments}(1),
		$\mbb{P}$-almost surely
		\[
		\lim_{d \to \infty}\sup_{(d,A)\in \Theta_\CW(p_d,d, M)} \abs{%
			H_\gamma \left[ \ESD\left(\rvW_\CW(A_d)\right), \mu_{W_\CW^\Box}(A) \right] %
			- H_\gamma \left[\mu_{W_\CW^\Box}(A_d), \mu_{W_\CW^\Box}(A) \right] %
		} = 0.
		\]
		\item [(SPN)]
		Under the settings (a)(b) of Proposition~\ref{prop_as_convergence_of_moments}(2),
		we have $\mbb{P}$-almost surely
		\begin{align}
		\lim_{d \to \infty}&\sup_{ \substack{ (A,\sigma)  \in \\  \Theta_\SPN(p_d,d,M) }} \abs{ H_\gamma\left[ \ESD \left(\rvW_\SPN(A_d,\sigma_0)\right), \mu_{W_\SPN^\Box(A, \sigma)} \right] %
			- H_\gamma \left[\mu_{W_\SPN^\Box(A_d, \sigma_0)}, \mu_{W_\SPN^\Box(A, \sigma)} \right] }  \\ 
		&= 0.
		\end{align}
	\end{enumerate}
\end{cor}

\begin{proof}
	By the boundedness of each parameter space, we have $\sup_d \sup_{ \theta \in \Theta_d}\abs{m_k(\mu_d^\Box (\theta))} < \infty$, where $\mu_d^\Box(\theta)$ is $\mu_{W_\CW^\Box}(A_d)$ (resp.\,$\Theta_d=\Theta_\CW(p_d, M)$) or $\mu_{W_\SPN^\Box}(A_d,\sigma)$ (resp.\,$\Theta_d=\Theta_\SPN(p_d,d,M)$).
	Then  gaps of moments converge to $0$ a.s. by Proposition~\ref{prop_as_convergence_of_moments}.
	Let $N \in \mf{F}$ with $\mbb{P}(N)=0$ such that the converges holds on $\Omega \setminus N$. Then for any $\omega \in \Omega \setminus N$, the samples of empirical distributions at $\omega$ satisfy the assumption of Lemma~\ref{lemma_moment_to_weak_in_cauchy}.
	Then the assertion follows from Theorem~\ref{thm_deterministic_convergence}.
\end{proof}

\begin{rmk}
	Haargerup-Thorbj{\o}rnsen \cite{haagerup2005new} shows bound of the variance of a function of random matrices.
	Unfortunately, since they consider a family of self-adjoint Gaussian random matrices, denoted by SGM in their paper, and because we treat a non-self-adjoint Gaussian random matrix $Z$,  evaluating the variance of $H_\gamma$ is out of the scope of their direct application.
\end{rmk}

\subsection{Iterative Methods for Cauchy Transforms }\label{ssection:iterative}
\hfill

\noindent In this section, we summarize iterative methods to compute possibly operator-valued  Cauchy transforms of FDE.

\bigskip
\bpara{$\mcR$-transform and Iterative Method I}\label{sssection:iterative-I}
The first method is based on the Voiculescu's \emph{$\mcR$-transform}.
The Cauchy transform of FDE is controlled by the  $\mcR$-transform. See \cite{voiculescu1992free} for the operator theoretic definition of scalar-valued $\mcR$-transform and \cite{Mingo2017free} for the Speicher's definition of scalar-valued and operator-valued $\mcR$-transform.
Here we introduce the \emph{operator-valued Cauchy transform}, which is useful to know the Cauchy transform of a matrix of circular elements.
\begin{defn}
	Let $(\mf{A}, \tau)$ be a C$^*$-probability space and $\mf{B}$ be a unital  C$^*$-subalgebra  of $\mf{A}$. Recall that they share the unit: $I_\mf{A} = I_\mf{B}$.
	\begin{enumerate}
		\item
		Then a linear operator $E \colon \mf{A} \to \mf{B}$ is called a
		\emph{conditional expectation onto $\mf{B}$} if it satisfies following conditions;
		\begin{enumerate}
			\item $E[b] = b$ for any $b \in \mf{B}$,
			\item $E[b_1 a b_2] = b_1E[a]b_2$ for any $a \in \mf{A}$ and $b_1, b_2 \in \mf{B}$,
			\item $E[a^*] = E[a]^*$ for any $a \in \mf{A}$.
		\end{enumerate}
		
		\item
		We write $\HP(\mf{B}) := \{ W \in \mf{B} \mid \text{ there is $\eps > 0$ such that $\Im W \geq \eps I_\mf{A}$} \}$ and $\HM(\mf{B}) := - \HP(\mf{B})$.
		\item
		Let $E \colon \mf{A} \to \mf{B}$ be a conditional expectation.
		For $a \in \mf{A}_\sa$, we define a \emph{$E$-Cauchy transform} as the map $\ccop{a}{E} \colon \HP(\mf{B}) \to \HM(\mf{B})$, where
		\[
		\ccop{a}{E}(Z) := E[ (Z - a)^{-1}],  \ Z \in \HP(\mf{B}).
		\]
		If there is no confusion, we also call $E$ a $\mf{B}$-valued Cauchy transform.
	\end{enumerate}
	
\end{defn}

Here we reformulate Helton-Far-Speicher \cite{helton2007operator}, which is an iterative method to compute operator-valued Cauchy transforms.

\begin{defn}
	Let $\eu{D}$ be a bounded domain (i.e.\,connected open subset) of a Banach space $E$ and $f:\eu{D} \to \eu{D}$.
	Then  $f$ maps $\eu{D}$ \emph{strictly into itself} if  there is $\eps > 0$ such that
	\begin{align}
	\inf \{ \norm{f(x) - y} \mid x \in \eu{D},  y \in E \setminus \eu{D}   \} > \eps.
	\end{align}
\end{defn}

\begin{prop}\label{prop:iterative}
	Let $\mf{B}$ be a C$^*$-subalgebra of $\mf{A}$.
	Assume that 
	\begin{align}\label{align:assumption-on-R}
	R \in \mr{Hol}(\HM(\mf{B})), \ \sup_{ B \in \HM(\mf{B}) \cap \oball{r}} \norm{ R(B) } < \infty \ ( r > 0), 
	\end{align}
	where  $\oball{r} := \{ B \in \mf{B} \mid \norm{B} < r \}$ for $r > 0$.
	For any fixed $Z \in \HP(\mf{B})$, let us define the map
	\[
	\eu{G}_Z(B):=(Z - R(B))^{-1}, B \in \HM(\mf{B}).
	\]
	Then
	\begin{enumerate}
		\item $\eu{G}_Z \in \mr{Hol}(\HM(\mf{B}))$.
		\item For any $B \in \HM(\mf{B})$, it holds that $\norm{\eu{G}_Z(B)} \leq \norm{ (\Image Z)^{-1} }$.
		\item For any $r > \norm { (\Image Z )^{-1} }$, the map $\eu{G}_Z$ sends $\oball{r} \cap \HM(\mf{B})$ strictly into  itself. 
		\item The equation $B = \eu{G}_Z(B)$  has a unique solution $B=G(Z)$.
		Moreover, for any $B_0 \in \HM(\mf{B})$ we have
		\[
		G(Z) = \lim_{ n \to \infty }\eu{G}_Z^n(B_0),
		\]
		where the convergence is in the operator norm topology.
	\end{enumerate}
\end{prop}

\begin{proof}
	Write $V = -iZ$, $W = iB$,  $\eta(W) = i R(-iW)$, and $\eu{F}_V(W) := (V + \eta(W))^{-1}$. Then
	$\eu{G}_Z(B) =  (iV - iR(B))^{-1} = i\eu{F}_V(iB)$, and 
	$\eu{G}_Z^{\circ n} (B) = i\eu{F}_V^{\circ n} (iB)$.
	Hence the assertion follows from \cite[Theorem~2.1, Proposition~3.21]{helton2007operator}.
\end{proof}

\bigskip
\bpara{Linearization Trick and Operator-valued Semicircular Elements}
To apply the above iterative method to SPN model, we use a linearization trick.
Firstly we embed the model into a square matrix. Then we use the operator-valued Cauchy transform defined as the following.

\begin{note}
	We denote by $O_{m,n}$ the $m \times n$ zero matrix for $m,n \in \N$ and  write $O_m := O_{m,m}$. We denote by $I_m$ the $m \times m$ identity matrix for $m \in \N$.     Let $ p , d \in \N $ with $p \geq d$.
	\begin{enumerate}
		\item    We denote by $\Lambda$ the map $M_{p,d}(\mf{A}) \to M_{p+d}(\mf{A})_\sa$ determined by
		\begin{align}\label{align:Lambda_A}
		\Lambda(X) :=
		\begin{pmatrix}
		O_{d} & X^* \\
		X & O_{p}
		\end{pmatrix}.
		\end{align}
		\item    We write
		\begin{align}
		Q :=
		\begin{pmatrix}
		I_{d} & O_{d,p} \\
		O_{p,d} & O_{p}
		\end{pmatrix}, \ Q^\bot := I_{p+d} - Q.
		\end{align}
		
		\item We denote by  $X_{+,+} \in M_{d}(\mf{A})$ (resp. $X_{-,- } \in M_p(\mf{A})$)  the $d \times d$ upper left  corner (resp. the $p \times p$ lower right corner) of $X \in M_{p+d}(\mf{A})$.
	\end{enumerate}
	
\end{note}

\begin{note}
	For any $k \in \N$, we write $E_k := \id_k \otimes \tau \colon M_k(\mf{A}) \to M_k(\C)$.
	Note that $E_k$ is a conditional expectation with $\tr_k \circ E_k = \tr_k \otimes \tau$.
	
\end{note}

\begin{prop}
	For any $z \in \HP{}$, we have
	$\ccop{X^*X}{E_d}(z) %
	= (1/\sqrt{z})\ccop{\Lambda(X)}{E_{p+d}}( \sqrt{z})_{+,+}$,
	where the branch of $\sqrt{z}$ is chosen as $\Re \sqrt{z} \geq 0$ and $\Image \sqrt{z} > 0$.
	
\end{prop}
\begin{proof}
	The proof is direct forward.
\end{proof}

Summarizing the above,
the computation of the operator-valued Cauchy transform of $W_\SPN^\Box(A,\sigma)$  is reduced to that of  $\Lambda(A) + \sigma \Lambda(C)$.

\begin{defn}\label{defn:linearized-model}
	Let us write $S:=\Lambda(C) $.
	The \emph{linearized FDESPN model} is the map  $W^\Box_\mr{lin} \colon M_{p,d}(\C) \times \R \to M_{p+d}(\mf{A})_\sa$, where
	\[
	W^\Box_\mr{lin}(A, \sigma) :=     \Lambda(A) + \sigma S = \Lambda(A) + \sigma \Lambda(C) .
	\]
\end{defn}

Let us review operator-valued semicircular elements. See \cite{Mingo2017free} for the detail.

\begin{defn}
	Let $E \colon \mf{A} \to \mf{B}$ be a conditional expectation.
	Let $S \in \mf{A}_\sa$ with $E[S] = 0$.
	We define the corresponding \emph{covariance mapping} $\eta :\mf{B} \to \mf{B}$ by
	\[
	\eta(B) := E[SBS].
	\]
	Then $S$ is called  \emph{ $E$-semicircular} if
	\begin{align*}
	\ccop{S}{d}(Z) =  \left[Z - \eta \left(\ccop{S}{d}(Z)\right) \right]^{-1} \ (Z \in \HP(\mf{B})).
	\end{align*}
\end{defn}

\begin{prop}
	It holds that $S:=\Lambda(C)$ is $E_{p+d}$-semicircular   and the corresponding covariance mapping is given by
	\begin{align}\label{align:covariance_map}
	\eta \big[  \begin{pmatrix}
	W_{11}&  W_{12}\\
	W_{12} & W_{22}
	\end{pmatrix} \big]
	= \begin{pmatrix}
	(p/d) \tr(W_{22}) I_d &  O_{d,p}\\
	O_{p, d} &  \tr(W_{11}) I_p
	\end{pmatrix}.
	\end{align}
	Since $\Lambda(A) + Z \in \HP(M_{d+p}(\C))$,
	it holds that
	$\ccop{ W^\Box_\mr{lin}}{E_{p+d}}(Z) =  \left[Z - \Lambda(A)  - \sigma^2 \eta \left(\ccop{W^\Box_\mr{lin}}{E_{p+d}}(Z)\right) \right]^{-1}.$
\end{prop}
\begin{proof}
	This is a direct consequence of \cite[Section~9.5]{Mingo2017free}.
\end{proof}

The time complexity of computing $\cc{W_\SPN^\Box}(z)$ through $\ccop{W^\Box_\mr{lin}}{E_{p+d}}(zI_{p+d})$ is at least $O(d^3)$ if  we compute in a naive way the inverse of $ (p+d) \times (p+d)$ matrix.
To reduce it, we consider another method to compute $\cc{W_\SPN^\Box}(z)$ based on the subordination.

\bigskip
\bpara{Iterative method II: Subordination}\label{sssec:iterative-II}
\begin{note}
	Let us write
	\begin{align}
	\mf{D}_2 := \C Q  \oplus \C Q^\bot =  \left\{\begin{bmatrix}
	xI_d & 0 \\
	0 & yI_p
	\end{bmatrix}
	\mid x,  y \in \C \right\} \subset M_{p+d}(\C).
	\end{align}
	We denote a conditional expectation $E_{\mf{D}_2} \colon M_{p+d}(\mf{A}) \to \mf{D}_2$ by
	\[E_{\mf{D}_2}(X) :=  \frac{\tau_{p+d}(QXQ)}{\tau_{p+d}(Q)}Q + \frac{\tau_{p+d}(Q^\bot XQ^\bot)}{\tau_{p+d}(Q^\bot)}Q^\bot =  \begin{bmatrix}
	\tr_{d} \otimes \tau(X_{ +  ,+ })I_d & 0 \\
	0 &  \tr_{p} \otimes \tau(X_{-,-})I_p
	\end{bmatrix}.
	\]
\end{note}

Note that $\mf{D}_2$ is $*$-isomorphic to $\C^2$.

\begin{defn}(Operator-valued Freeness)
	Let $(\mf{A}, \tau)$ be a C$^*$-probability space,  and $E :  \mf{A} \to \mf{B}$ be a conditional expectation.
	Let  $(\mf{B}_j)_{j \in J}$ be a family of $*$-subalgebras of $\mf{A}$ such that $\mf{B} \subset \mf{B}_j$.
	Then $(\mf{B}_j)_{j \in J}$  is said to be \emph{$E$-free}
	if the following factorization rule holds: for any $n \in \N$ and indexes $j_1, j_2, \dots, j_n \in J $ with $j_1 \neq j_2 \neq j_3  \neq \cdots \neq j_n$, and $a_l \in \mf{B}_l$ with $E(a_l)=0$ $(l = 1 ,\dots, n)$, it holds that
	\[
	E(a_1 \cdots a_l) = 0.
	\]
	In addition, a family of elements $X_j \in \mf{A}_\sa \ ( j \in J)$ is called $E$-free if the family of $*$-subalgebra of the $\mf{B}$-coefficient polynomials of $X_j$ is $E$-free.
	
\end{defn}

Here we summarize observations about $\mf{D}_2$-valued freeness. See \cite[Section~9.2]{Mingo2017free} for the definition of operator-valued free cumulants.
\begin{prop}
	The operator $S:=\Lambda(C) \in M_{p+d}(\mf{A})$ is  $E_{\mf{D}_2}$-semicircular.
	Its covariance mapping $\eta_2 \colon \mf{D}_2 \to \mf{D}_2$ is given by
	\[
	\eta_2(xQ + y Q^\top) := E_{\mf{D}_2}(S (xQ + y Q^\top) S ) = \begin{bmatrix}
	(p/d)yI_d & 0 \\
	0 & xI_p
	\end{bmatrix}.
	\]
	Moreover, the pair $(S, \Lambda(A))$ is $E_{\mf{D}_2}$-free for any $A \in M_{p,d}(\C)$.
\end{prop}
\begin{proof}
	Since $\eta(\mf{D}_2) \subset \mf{D}_2$, and by \cite[Section~9.4,  Corollary~17]{Mingo2017free}, it holds that $S$ is $\mf{D}_2$-semicircular, with covariance mapping given by $\eta_2 = \eta |_{\mf{D}_2}$.
	
	In addition, each $\mf{D}_2$-valued free cumulant of $S$ (see \cite[Section~9.2]{Mingo2017free} for the definition) is given by  the restriction of its $M_{p+d}(\C)$-valued free cumulants.
	On the other hand, all $M_{p+d}(\C)$-valued free cumulants of $\Lambda(A)$ vanish since $\Lambda(A) \in M_{p+d}(\C)$ and $E^\mf{B}[\Lambda(A)]=0$. Moreover, $M_{p+d}(\C)$-valued mixed free cumulants of $(S, \Lambda(A))$ vanish because of $M_{p+d}(\C)$-freeness.
	Summarizing above, the restriction to $\mf{D}_2$ of each $M_{p+d}(\C)$-valued cumulants of $(S, \Lambda(A))$ belong to $\mf{D}_2$. Hence by \cite[Theorem~3.1]{nica2002operator}, each $\mf{D}_2$-valued free cumulants of $(S, \Lambda(A))$  is equal to the restriction to $\mf{D}_2$ of each $M_{p+d}(\C)$-valued free cumulants.
	In particular, all mixed $\mf{D}_2$-valued free cumulants of  $(S, \Lambda(A))$  vanish, which implies $E_{\mf{D}_2}$-freeness.
\end{proof}

By this proposition,  for any $Z \in \HP(\mf{D}_2) \simeq \HP(\C^2)$,
$\ccop{\sigma S}{E_{\mf{D}_2}}(Z) := E_{\mf{D}_2}[(Z - \sigma S)^{-1} ]$ is given by
\begin{align}
\ccop{\sigma S}{E_{\mf{D}_2} }(Z) = \lim_{ n\to \infty} \mcG_{Z, \sigma}(B_0),
\end{align}
where $\mcG_{Z, \sigma}(B):= ( Z - \sigma^2 \eta_2(B) )^{-1}.$

By the $\mf{D}_2$-valued freeness, we can use the following subordination method by Belinschi-Mai-Speicher \cite[Theorem~2.3]{belinschi2013analytic}.

\begin{prop}\label{prop:subordination}
	Let $(\mf{A}, \tau)$ be a C$^*$-probability space, and $\mf{B}$ be a   unital C$^*$-subalgebra, and $E \colon \mf{A} \to \mf{B}$ be a conditional expectation.
	We define the \emph{$h$-transform} of $a \in \mf{A}_\sa$ with respect to $E$ by the map $h_a \colon \HP(\mf{B}) \to \overline{\HP(\mf{B})}$ with $h_a(B) = \ccop{a}{\mf{B}}(B)^{-1} - B$.
	
	Assume that $(x,y)$ is a $E$-free pair of self-adjoint elements in $\mf{A}$.  Write 
	\[
	\Psi(B,  Z) = h_{y}(h_{x}(B) + Z) + Z,
	\]
	and $\Psi_Z(B) := \Psi(B, Z)$.
	Then there is $\psi\in \mr{Hol}(\HP(\mf{B}))$ so that for all $Z \in \HP(\mf{B})$,
	\begin{enumerate}
		\item $\psi(Z) = \lim_{ n \to \infty }\Psi_Z^{\circ n}(B_0)$ for any $B_0 \in \HP(\mf{B})$,
		\item $\Psi_Z(\psi(Z)) = \psi(Z)$,
		\item  $\ccop{x+y}{E}(Z) = \ccop{x}{E}(\psi(Z))$.
	\end{enumerate}
	In addition, for any fixed $Z \in \HP(\mf{B})$ and $\eps > 0$ with $\Im Z > \epsilon I_d$, there is $m > 0$ depending on $Z, x$ and $y$ so that
	\[
	\Psi_Z(\HP(\mf{B}))\subset U(0,m ) \cap (\HP(\mf{B}) + i\frac{\eps}{2}) \subsetneqq \mc{D}_{Z,x,y}:=U(0,2m) \cap (\HP(\mf{B}) + i \frac{\eps}{2}).\]
	In particular, $\Psi_Z(\mc{D}_{Z, x,y}) \subsetneqq \mc{D}_{Z, x,y}$.
\end{prop}

\begin{proof}
	This is a direct consequence of \cite[Theorem~2.3]{belinschi2013analytic} and the claim in its proof.
\end{proof}

Now we have another representation of the scalar-valued Cauchy transform of $W_\SPN^\Box$ as the following.

\begin{cor}\label{cor:fwd_SPN_sub}
	We have $\HP(\mf{D}_2)= \{ x, y \in \HP{} \mid xQ + y Q^\bot \}$.
	Fix $p, d \in \N$ with $p \geq d$.
	Let $a \in \R^d$ and $\sigma \in \R$.
	Write $G_a:= \ccop{\Lambda(\iota^d_{p,d}(a))}{\mf{D}_2}$, $h_a := h_{\Lambda(\iota^d_{p,d}(a))}$, $G_\sigma := \ccop{\sigma S}{\mf{D}_2}$, $h_\sigma := h_{\sigma S}$, where $S = \Lambda(C)$.
	Moreover let us define a map $\Psi \colon \HP(\mf{D}_2) \times \HP(\mf{D}_2) \times \R^d \times \R \to \HP(\mf{D}_2) \to \HP(\mf{D}_2)$ by
	\begin{align}
	\Psi(B, Z, a, \sigma ) := h_a( h_\sigma (B) + Z ) + Z,
	\end{align}
	We write $\Psi_Z(B, a, \sigma ) := \Psi(B, Z, a, \sigma )$.
	Then the limit
	\begin{align}
	\psi(Z, a, \sigma) := \lim_{n \to \infty} \Psi_{Z, a, \sigma}^n(B_0),
	\end{align}
	exists, and it is independent from the choice of the initial point $B_0 \in \HP(\C^2)$. Moreover,
	\begin{align}
	\ccop{W^\Box_\mr{lin}(\iota^d_{p,d}(a), \sigma)}{\mf{D}_2}(Z) &= \ccop{ \Lambda( \iota^d_{p,d}(a)) + \sigma S }{\mf{D}_2}(Z) = \ccop{\sigma }{\mf{D}_2}(\psi (Z, a, \sigma )) \ ( Z \in \HP(\mf{D}_2) ),\\
	\cc{W_\SPN^\Box(\iota^d_{p,d}(a), \sigma)}(z) &=\frac{1}{\sqrt{z}} \ccop{\sigma}{\mf{D}_2}\left(\psi (\sqrt{z}I_2, a, \sigma ) \right)_{+,+} \ ( z \in \HP{}).
	\end{align}
\end{cor}

\begin{proof}
	This follows  immediately  from Proposition~\ref{prop:subordination}.
\end{proof}

\begin{rmk}
	Note that the  method described in Corollary~\ref{cor:fwd_SPN_sub} is performed in $\C^2$ under the $*$-isomorphism    $\mf{D}_2 \simeq \C^2$.
	In addition, this method requires two nested loops of the computation of $\ccop{a}{\mf{D}_2}(B)=\ccop{\Lambda(\iota^d_{p,d}(a))}{\mf{D}_2}(B)$ $(B \in \HP(\mf{D}_2) )$.
	We note that the time complexity of the computation of $\ccop{a}{\mf{D}_2}(B)$ is $O(d)$;
	\begin{align}
	\ccop{a}{\mf{D}_2} \left( B\right) = \frac{b_2}{d}\sum_{k=1}^d \frac{1}{b_2b_1 - a_k^2} Q + %
	\left[\frac{b_1}{p} \sum_{k=1}^{d}\frac{1}{b_2b_1 - a_k^2} + \frac{p-d}{pb_2} \right]Q^\bot,
	\end{align}
	where $B = b_1Q + b_2 Q^\bot$, and  $\Im b_1 , \Im b_2 >0$.
\end{rmk}

\subsection{Gradients of Cauchy Transforms}\label{ssection:gradient}
\hfill

\noindent We discuss the gradients of operator-valued Cauchy transforms of FDE with respect to parameters.  

\begin{defn}
	Let  $\mf{A}_1$ and $\mf{A}_2$ be C$^*$-algebras and $\eu{D}_j \subset \mf{A}_j$ $(j=1,2)$ be domains. Then a map $F \colon \eu{D}_1 \to \eu{D}_2$ is called $\emph{holomorphic}$ if for each $a \in \eu{D}_1$, there is a unique bounded linear map $D_a F \colon \mf{A}_1 \to \mf{A}_2$  such that 
	\[
	\lim_{ x \in \eu{D}_1, x \neq a, \norm{x - a}\to 0} \frac{\norm{F(x) - F(a) - D_a F ( x -a )} }{ \norm{x -a}} =0.
	\]
	
	In addition, we write  $\mr{Hol}(\eu{D}_1, \eu{D}_2) :=  \{  F \colon \eu{D}_1 \to \eu{D}_2 \mid \text{holomorphic}\}$ and 
	$\mr{Hol}(\eu{D}_1):=\mr{Hol}(\eu{D}_1,\eu{D}_1)$.
\end{defn}

In this section, we fix a finite dimensional C$^*$-algebra $\mf{B}$ and write 
$\oball{r} = \{ b \in \mf{B} \mid \norm{b} < r\}$.

The following lemma is pointed out by Genki Hosono.
\begin{lem}\label{lem:energy}
	Let $\mc{D}$ be a bounded domain in the finite dimensional C$^*$-algebra $\mf{B}$ and  $f \in \mr{Hol}(\mf{B}, \mc{D})$.
	Assume that $f$ has a unique fixed point $\psi \in \mc{D}$ and  $\lim_{n \to \infty}f^{\circ n}(z) = \psi$ for any $z \in \mf{B}$.
	Then $ \| D_\psi f  \| < 1$.
	In particular, $I - D_\psi f $ is invertible, where $I$ is the identity map.
\end{lem}

\begin{proof}
	Without loss of generality, we may assume that $\psi = 0$, and in particular $f(0)=0$.
	Let $\xi \in \mf{B}$ be an eigenvector of $D_0 f$ and $\lambda$ be the corresponding eigenvalue.
	Fix $r> 0$ such that $U(0,2r) \subset \mc{D}$.
	Then by Cauchy's integral formula it holds that
	\begin{align}\label{eq_cauchy_integration}
	(D_0 f)^n \xi  =\lim_{t \to  0}\frac{f^{\circ n}(t \xi)}{t} = \frac{1}{2\pi}\int_{z \in \C, |z |=r } \frac{f^{\circ n}(z\xi)}{z^2}dz, \ n \in \N.
	\end{align}
	Since the sequence $f^{\circ n}$ is uniform bounded and converges to $0$ as $n \to \infty$ at every point, the   right hand side of \eqref{eq_cauchy_integration} converges to $0$ by the bounded convergence theorem.
	Since $D_0(f^n) = (D_0 f)^n$, the left hand side of \eqref{eq_cauchy_integration} is equal to $\lambda^n \xi$.
	Hence $|\lambda| < 1$.
	Since $\mf{B}$ is finite-dimensional, the spectral norm is equal to the maximum of the absolute value of eigenvalues,  which proves the assertion.
\end{proof}

\begin{thm}\label{thm:gradient}
	Let $\Theta \subset \R^m$ be a non-empty open subset, and
	$\mcR : \HM(\mf{B}) \times \Theta \to \HM(\mf{B})$.
	For $\theta \in \Theta$, let us write $\mcR_\theta:=\mcR(\cdot, \theta)$.
	Assume that following conditions:
	\begin{enumerate}
		\item Fix $\theta \in \Theta$. Then $\mcR_\theta \in \mr{Hol}(\HM(\mf{B}))$
		and bounded on bounded subsets; for any $r >0$,
		\[
		\sup \{ \norm{\mcR_\theta(B)}  \mid B \in \HM(\mf{B}) \cap \oball{r} \} < \infty ,
		\]
		
		\item  the map $\theta \mapsto \mcR( B, \theta)$ is of class C$^1$ for any $B \in \HM(\mf{B})$.
	\end{enumerate}
	We define a map  $\mcG : \HM(\mf{B}) \times \Theta \times \HP(\mf{B})  \to \HM(\mf{B})$ by
	\[
	\mcG(B, \theta, Z) := (Z - \mcR_\theta(B))^{-1},
	\]
	where we denote by $(B, \theta, Z)$ the canonical coordinate on $\HM(\mf{B}) \times \Theta \times \HP(\mf{B})$.
	We write $\mcG_{\theta, Z}(B): = \mcG(B,\theta,Z)$.
	Assume that $G_\theta(Z)$ is the solution in $\HP(\mf{B})$ of the equation $\mcG(W, \theta, Z) = W$.
	Then $||D_{G_\theta(Z)}\mcG_{\theta, Z}|| < 1$  and
	\begin{align}
	\del{G_\theta}{\theta}(Z) = (I - D_{G_\theta(Z)}\mcG_{\theta, Z})^{-1} \del{\mcG}{\theta}( G_\theta(Z), \theta, Z).
	\end{align}
	
\end{thm}

\begin{proof}
	Fix $\theta$ and $Z$. Then by Proposition~\ref{prop:iterative}, $\mcG_{\theta, Z}$ satisfies the conditions in Lemma~\ref{lem:energy}, which proves the assertion.
\end{proof}

By the same lemma, we show that the following theorem about the gradient of the subordination.
\begin{thm}\label{thm:norm_D_sub}
	Under the setting of Proposition~\ref{prop:subordination} with the assumption that $\mf{B}$ is finite dimensional, it holds that
	$\norm{ D_{\psi(Z)} \Psi_{Z}} < 1$ for any $Z \in \HP(\mf{B})$.
\end{thm}
\begin{proof}
	Fix $Z$. Then by Proposition~\ref{prop:subordination}, $\Psi_Z$ satisfies the conditions in Lemma~\ref{lem:energy}, which proves the assertion.
\end{proof}

\begin{cor}\label{cor:derivation_omega}
	Under the setting of Corollary~\ref{cor:fwd_SPN_sub}, we have
	\begin{align}
	\del{\psi(Z)}{a_k} &= (1 - D_{\psi(Z)} \Psi_Z)^{-1} \cdot \del{h_a}{a_k}(\psi(Z)), \ k= 1, \dots, d, \\
	\del{\psi(Z)}{\sigma} &= (1 - D_{\psi(Z)} \Psi_Z)^{-1} \cdot D_{\psi(Z)} h_a \cdot  \del{h_{\sigma S}}{\sigma}(\psi(Z)),
	\end{align}
	for any $B \in \HP(\mf{D}_2)$, $a \in \R^d$, and $\sigma \in \R$, where
	$h_a$ is the $\mf{D}_2$-valued h-transform of $A = \iota^d_{p,d}(a)$.
\end{cor}

\begin{proof}
	Recall that $\mf{D}_2$ is finite-dimensional.
	The assertion is a direct consequence  of Theorem~\ref{thm:norm_D_sub} and Proposition~\ref{cor:fwd_SPN_sub}.
\end{proof}

\subsection{Proof of Main Theorem}\label{sscec:proofs-of-main}

\begin{proof}[Proof of Theorem~\ref{thm:main_CW}]
	Define $\mu_\CW^\Box( \theta) := \mu_{W^\Box_\CW(\theta)}$ for $\theta \in \Theta_\CW(p,d)$ and $p,d \in \N$ with $p \geq d$.
	Then the assertion (1) follows from Corollary~\ref{cor_cnl_converges}.
	The Cauchy transform $\cc{W_\CW^\Box(\iota(a))}(z)$ is the solution  of the equation $G = \mcG(G, z,a)$ in the variable $G$ by \cite[Section~6.1]{calros2015free}. 
	
	Next we claim that $R$ defined in (2) satisfies the assumption \eqref{align:assumption-on-R}.
	Fix $v \in \R^p$.
	For any $b \in \HM{}$,  write $b = x + iy$ where $x, y \in \R$ with $y < 0$.
	Then
	\begin{align}
	\Im \mcR(b,v) = \frac{1}{d}\sum_{k=1}^p v_k  \frac{ v_k y }{ (1 - v_k x)^2 + v_k^2y^2 } = \left[ \frac{1}{d}\sum_{k=1}^p\frac{ v_k^2 }{ (1 - v_k x)^2 + v_k^2y^2 } \right] y < 0.
	\end{align}
	It holds that
	\begin{align}
	\abs{ \mcG(b,z,v)} \leq  \frac{1}{\Im z}.
	\end{align}
	Pick $r, \delta > 0$, with
	\begin{align}
	\frac{1}{\Im z } < r- \delta < r < r + \delta  <  \min \{ \frac{1}{v_k} \mid k= 1, \dots, p \}.
	\end{align}
	Then
	\begin{align}
	\abs { \abs{ \mcG(b,z,v)} - g } > \delta,
	\end{align}
	for any $g \in \C$ with $ \abs{g} = r$.
	In addition, for any $b \in \HM{} \cap U(0,r)$,
	\begin{align}
	\abs{v_k^{-1} - b} \geq \delta.
	\end{align}
	Thus we have
	\begin{align}
	\sup\{ \abs{R(b)}  \mid  b \in \HM{}   \cap U(0,r) \} \leq \frac{p}{d\delta} < \infty.
	\end{align}
	Set
	\begin{align}
	m_r = \abs{z} +      \sup\{ \abs{R(b)}  \mid  b \in \HM{}   \cap U(0,r) \} < \infty.
	\end{align}
	Then
	\begin{align}
	- \Im \mc{G}(\cdot, z, v) \geq \frac{\abs{\Im z}}{m_r^2 } > 0.
	\end{align}
	Therefore, $\mc{G}(\cdot, z, v)$ maps $\HM{}   \cap U(0,r)$ strictly into itself: set $\eps := \min(\delta, \Im z /m_r^2)$ then
	\begin{align}
	\mc{G}\left( \HM{} \cap U(0,r), z, v\right) + U(0,\eps) \subset U(0,r) \cap \HM{} .
	\end{align}
	Hence the claim is proven and the assertion (2) follows from Proposition~\ref{prop:iterative}.
	
	The assertion (3) directly follows from (2).
	
	For any $z \in \HP$ and $A_0 \in M_p(\C)$, there is a domain $\eu{D} \subset M_p(\C)$ such that the map $A \mapsto (z - W_\CW^\Box(A))^{-1}$ is holomorphic on $\eu{D}$. 
	In addition, by Theorem~\ref{thm:gradient},  (4) follows.
\end{proof}

\begin{proof}[Proof of Theorem~\ref{thm:main_spn}]
	Define $\mu_\SPN^\Box (\theta) := \mu_{W_\SPN^\Box(\theta)}$.
	The assertion (1) follows from Corollary~\ref{cor_cnl_converges}, and (2) follows from Corollary~\ref{cor:fwd_SPN_sub} and the identification by the $*$-isomorphism  $\mf{D}_2 \simeq \C^2$.
	The assertion    (3) directly follows from (2).
	For any $z \in \HP$,  $ \theta_0 \in \Theta_\SPN(p,d)$ and $p, d \in \N$ with $ p> d$, there is a domain $\eu{D} \subset \Theta_\SPN(p,d)$ such that $\theta_0 \in \Theta_\SPN(p,d)$ and  the map $ \theta \mapsto (z - W_\SPN^\Box(\theta))^{-1}$ is holomorphic on $\eu{D}$. 
	In addition, by Theorem~\ref{thm:gradient} and by Theorem~\ref{thm:norm_D_sub},   (4) follows.
\end{proof}

\section{Experiments and Discussion}

\bpara{Implementation Detail}
First, we discuss numerical considerations.

The first one is about the iterative method I (see Section~\ref{sssection:iterative-I}).
As proposed in \cite{helton2007operator}, when we compute the Cauchy transform (or matrix-valued one) by the iteration described in  Proposition~\ref{prop:iterative}, we replace the map 
$B \mapsto \mcG_Z(B)$ by the averaged version $B \mapsto \widetilde{\mcG}_Z(B):= B/2 + \mcG_Z(B)/2$. We observed the speed up of the convergence in our examples FDECW model and FDESPN model by using the averaging. 
We continue  the iterates  while the difference is not small ;
$\norm{\widetilde{\mcG}_Z^{n+1}(G_0) - \widetilde{\mcG}_Z^n(G_0) }_2 > \eps$,
where the norm is Euclid norm and $\eps > 0$ is a given threshold. 
We set $\eps := 10^{-8}$.
In our algorithm we have to use the iterative method I for many values $Z_n = z_n I_d$ $(n=1, \dots, k)$ and $d \in \N$ is the dimension of (operator-valued) Cauchy transform. For the speed-up, we use $G(Z_n)$ as the initial value to compute $G(Z_{n+1})$, and  use the initial value $G_0 := -i I_d $ to compute $G(Z_1)$.

The second one is about the iterative method II (see Section~\ref{sssec:iterative-II}), that is, the subordination method.
We do not use the averaging for iterates of $\Psi$ in the subordination method. 
We continue the iterates while the difference is not small, i.e., its Euclid norm is larger than $\eps = 10^{-8}$.
We set the initial value used for the map $\Psi$ to be $iI_d$.
In particular, for SPN model, we set it $(i,i) \in \C^2$. 
Third,  for Adam (see Section~\ref{ssec:optimization}), we set $\alpha = 10^{-4}$, $\beta_1 = 0.9$, $\beta_2 = 0.999$, $\eps = 10^{-8}$ as suggested in \cite{kingma2014adam}.

\subsection{Optimization without Regularization}
\hfill

\noindent In this section, we show numerical results of optimization of CW and SPN under some values of the scale parameter $\gamma$.
Throughout experiments, each sample is a single-shot observation generated from a true model.

\newcommand{\sort}{\mr{sort}}

\bpara{Assumption} We assume the following conditions.

\begin{enumerate}
	\setlength{\leftskip}{4mm}
	
	\item [(CW1)] $p,d = 50$.
	\item [(CW2)]  Each true parameter $b_\mr{true} \in \R^d$ of CW model is generated uniformly from $[-0.1, 0.1]^p$.
	\item [(CW3)] We use the parameter space $\Xi_\CW(p,d, M)$  with $M=1$.
	\item [(CW4)]  The initial value of the parameter $b \in \R^p$  is generated uniformly from $[- 1/ \sqrt{p}, 1/ \sqrt{p}]^p$.
	
\end{enumerate}

\begin{enumerate}
	\setlength{\leftskip}{5mm}
	\item [(SPN1)] $p,d = 50$.
	\item [(SPN2)] Each sample is generated from  $W_\SPN(A_\mr{true}, \sigma_\mr{true})$ over $\R$, where $\sigma_\mr{true}=0.1$ and $A_\mr{true}$ is drawn from 
	$A_\mr{true} = U D V$such that
	\begin{enumerate}
		\item $U \in M_p(\R), V \in M_d(\R)$ are drawn independently from the uniform distribution on the orthogonal matrices,
		\item $D \in M_{p,d}(\R)$  is a rectangular diagonal matrix given by the array $a_\mr{true} \in \R^d$ which is uniformly generated from $[0, 1]^d$.
	\end{enumerate}
	\item [(SPN3)] We use the parameter space $\Xi_\SPN(p,d, 1.2)$.
	\item [(SPN4)] As the initial values, we set $\sigma = 0.2$, and set $a$ to the vector of eigenvalues of the sample matrix.
\end{enumerate}    
\bigskip

\bpara{Validation Loss} To evaluate the optimized parameter, we use the validation loss defined as follows.
\begin{align}
V_\CW(b, b_\mr{true}) &:= \norm{ \sort(b) - \sort(b_\mr{true})}_2,    \\        V_\SPN((a,\sigma), (a_\mr{true}, \sigma_\mr{true})) &:= \norm{ \sort(a) - \sort(a_\mr{true})}_2 + \abs{\sigma - \sigma_\mr{true}},
\end{align}
where $\sort(v) \in \R^p$ is the sorted vector of $v \in \R^p$ in ascending order, $b$ is the estimated parameter. Here we compare sorted vectors because eigenvalue distributions are stable under any permutations.
\begin{figure}[htbp]
	\ifthenelse{\boolean{DVIPDF}}{}{
		\centering
		\includegraphics[width=0.49\linewidth]{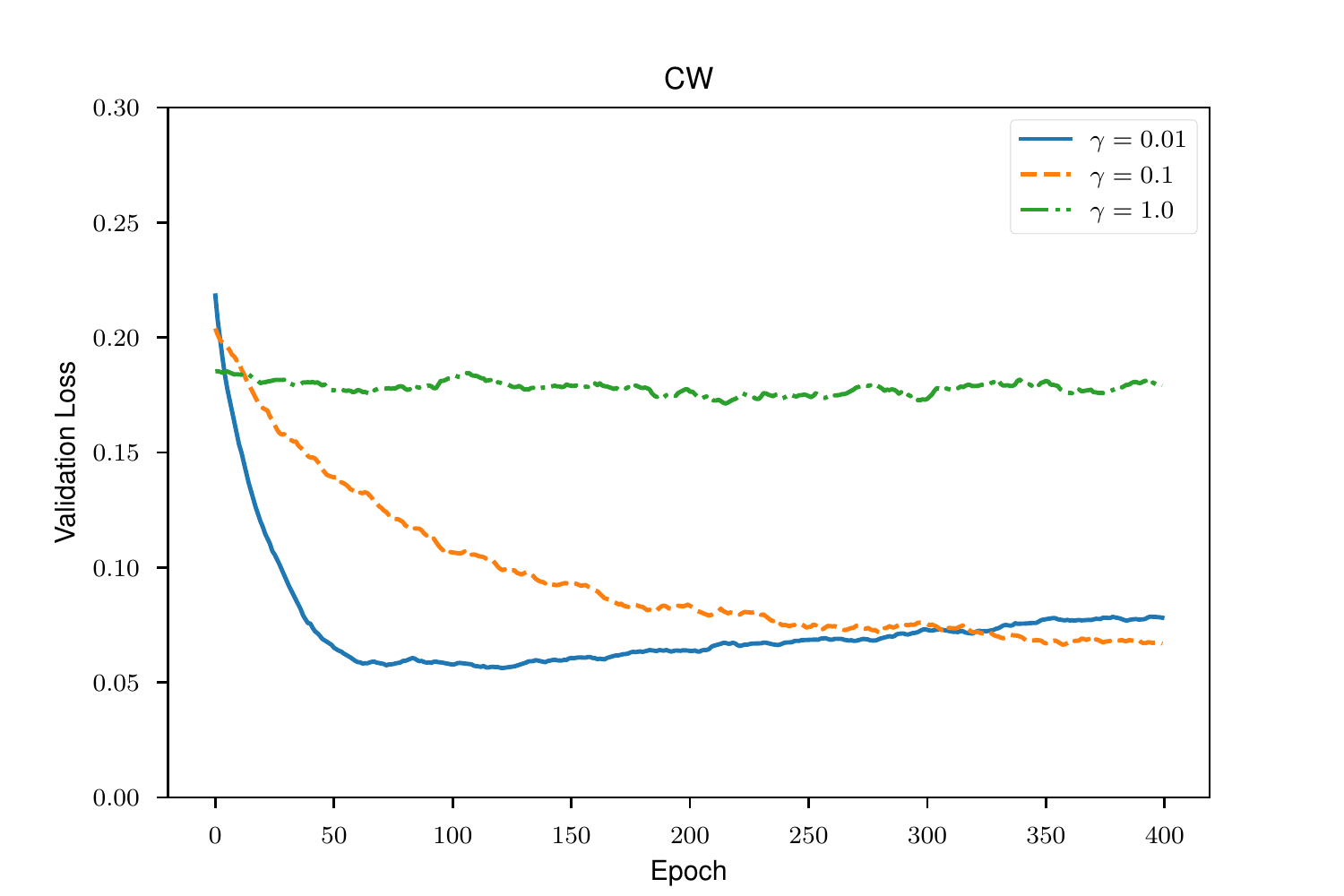}    
		\includegraphics[width=0.48\linewidth]{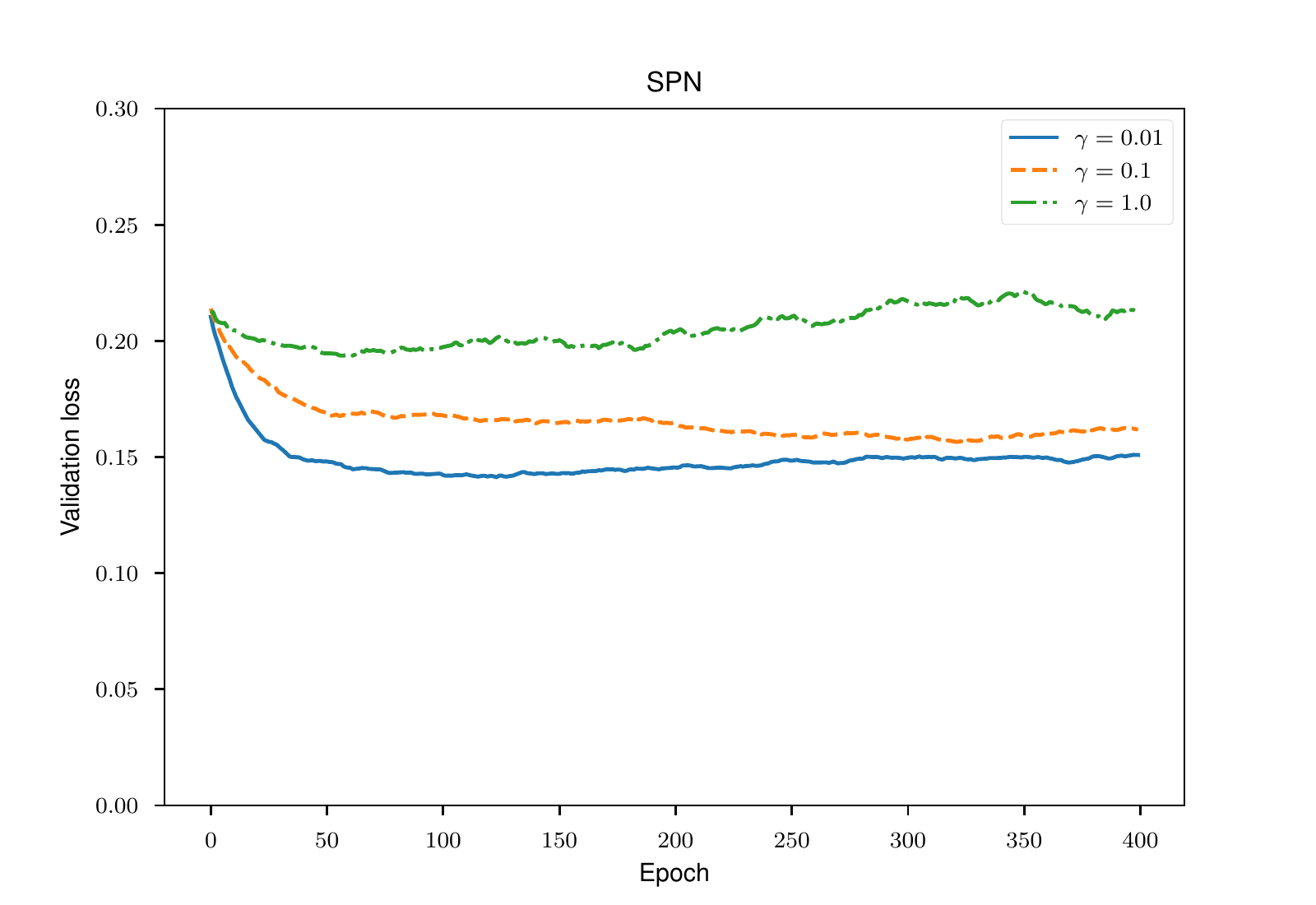}
	}
	\caption{Validations loss curves of CW model (left) and that of SPN model (right). We set the scale $\gamma=0.01,0.1$ and  $1$.  Each curve is the average of $10$-experiments.}    
	\label{fig:val-cw-spn}
\end{figure}

Figure \ref{fig:val-cw-spn} shows the optimization results of CW model and SPN model. 
The horizontal axis indicates the number of optimization epochs.
We set the max iteration as  $N = 400 d$ for both models.
The vertical axis indicates the validation loss  $V_\CW$ and  $V_\SPN$.

There was a  difference between the scales;  the smaller the scale became, the faster the validation loss decreased.  
For the large scale, the convergence speed became slow, or the validation loss did not converge.
The significant finding is that the values of the validation loss at the stationary points did not become different so much between $\gamma= 0.1 $ and $0.01$.
As a consequene,  the parameter A does not need to be too small.

\begin{table}[htbp]
	\caption{The total number of iterations of $\mc{G}$ per step averaging over $N=2.0 \times 10^5$ steps.  Each value is averaged over $10$ experiments with the sample standard deviation. }
	\label{table:iteration}
	\begin{tabular}{|c||c|c|c|}
		\hline
		$\gamma$ & $0.01$ & $0.1$ & $1$ \\
		\hline 
		$\CW (p=d=50)$ & $ 3.5 \times 10$ $(\pm 1.1)$ & $2.7 \times 10$ $(\pm 0.1 )$ & $2.6 \times 10$ $(\pm 0.0)$ \\
		$\CW (p=d=200)$ & $ 3.5 \times 10$ $(\pm 1.1)$ & $2.7 \times 10$ $(\pm 0.1 )$ & $2.6 \times 10$ $(\pm 0.0)$ \\
		\hline
		$\SPN(p=d=50)$ & $1.4 \times 10^3$ $(\pm 4.2 \times 10^2)$ & $1.3\times 10^2$ $(\pm 3.9 \times 10 )$ & $ 7.5 \times 10$ $(\pm 0.6 )$\\
		$\SPN(p=d=200)$ & $ 0.5 \times 10^3$ $(\pm 1.5 \times 10^2)$ & $1.0\times 10^2$ $(\pm 1.7 )$ & $ 7.5 \times 10$ $(\pm 0.4 )$\\
		\hline
	\end{tabular}
\end{table}

Now, Table~\ref{table:iteration} shows that 
the number of iterations for computing Cauchy transforms increased as the scale was set small.
Besides, the number corresponding to the SPN model increased faster than that corresponding to the CW model.
Therefore, it turned out that too small $\gamma$ is not suited for the SPN model.
However, the numbers for both CW and SPN models did not increase as the dimension increased. 
Further investigation is required to find the ideal way to choose the scale parameter $\gamma$.

Note that we compute the Cauchy noise loss for SPN and CW models by iterative mappings on one or two-dimensional complex vector space,
and we compute their gradients  by the implicit differentiation.    
Therefore,  each step of the iterative methods and computing gradients require low time complexity concerning $d$ and $p$.

\subsection{Dimensionality Recovery}
\hfill

\noindent In this section, we show the dimensionality recovery method based on the optimization of SPN model by using the Cauchy noise loss with a regularization term.
We assume the followings with (SPN3) and (SPN4).

\begin{enumerate}
	\item [(SPN1')] $(p,d)= (50, 50)$ or $(100, 50)$.
\setlength{\leftskip}{5mm}

	\item [(SPN2')]
\setlength{\leftskip}{5mm}

	\begin{enumerate}
\setlength{\leftskip}{5mm}

		\item $U \in M_p(\R), V \in M_d(\R)$ are drawn independently from the  uniform distribution on the orthogonal matrices,
		\item $D \in M_{p,d}(\R)$  is a rectangular diagonal matrix given by the array 
		\[(0 ,\dots, 0, x_1, \dots, x_{d_\mr{true}}) \in \R^d,\]
		where $x_1, \dots, x_{d_\mr{true}}$ are generated independently from the uniform distribution on $[\lambda_\mr{min}, 1]$,
		$d_\mr{true} = 10,20,30,40$, and $\lambda_\mr{min} = 0.05, 0.1, 0.15, 0.2, 0.3, 0.4$.
	\end{enumerate}

\end{enumerate}
From a single-shot sample matrix, we estimate $d_\mr{true}$.
The dimensionality recovery based on the Cauchy noise loss is as follows.
To shrink small parameters to zero, we add  the $L^1$ regularization term of $a$, defined as the following,  to the Cauchy noise loss.
\begin{align}
\norm{a}_1:=  \sum_{i=1}^{d}\abs{a_i}. 
\end{align}
In our experiments, we fixed $\xi= 10^{-3}$.
First, we optimize the parameters $(a, \sigma)$ based on Algorithm~\ref{alg:BOGD-FDE} by using the Cauchy noise loss with the regularization term defined as 
\begin{align}
L_{\gamma, \xi}(a,\sigma, \lambda) := L_\gamma( \mu_\SPN^\Box(a,\sigma), \lambda) + \xi \norm{a}_1,
\end{align} 
where $\xi > 0$, instead of by using the Cauchy nose loss itself.
Lastly, the rank is estimated as $\#\{ j \mid a_j < \xi\}$.

\begin{figure}[htbp]
	\ifthenelse{\boolean{DVIPDF}}{}{
		\centering
		\includegraphics[width=0.495\linewidth]{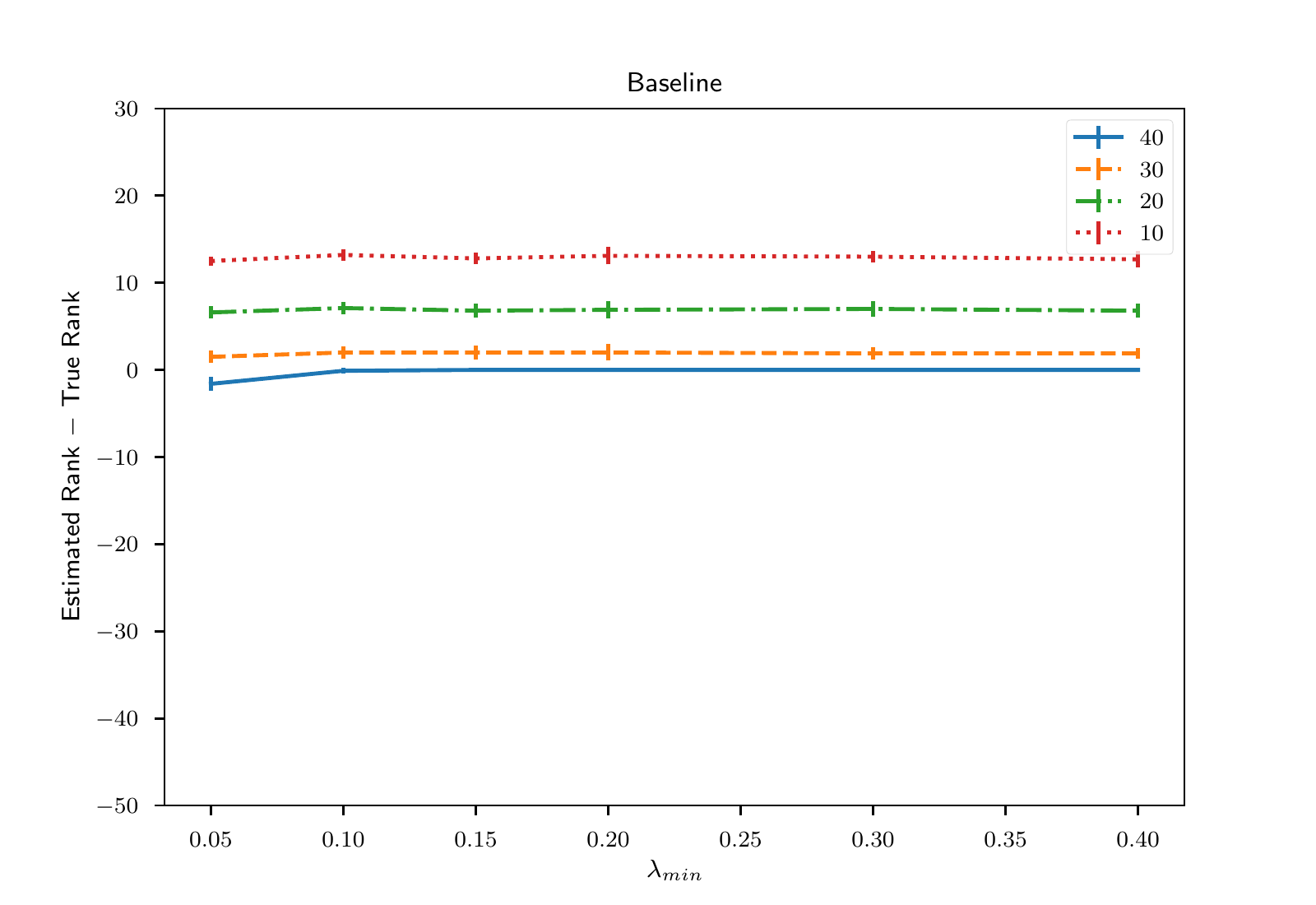}    
\includegraphics[width=0.495\linewidth]{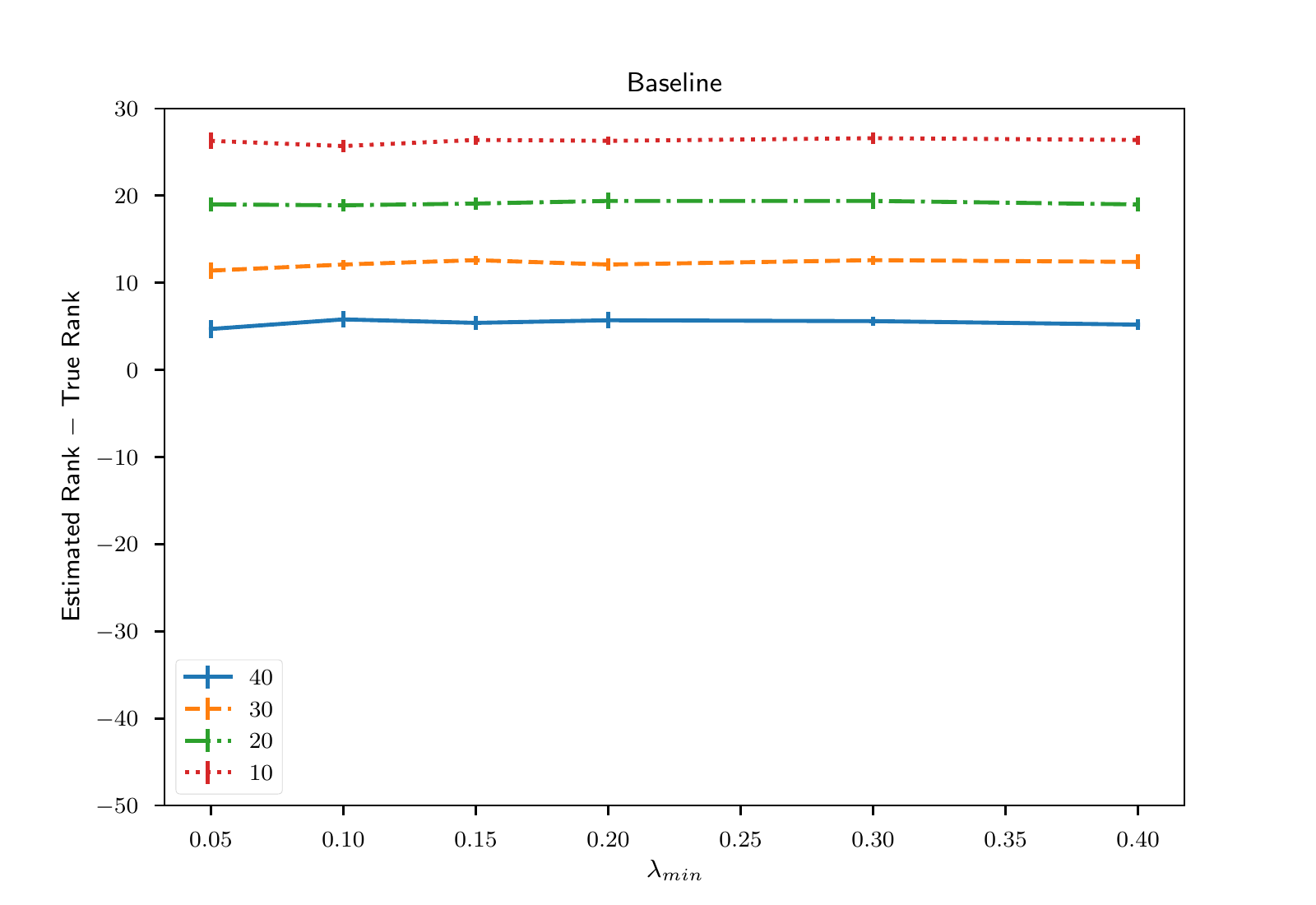}

\includegraphics[width=0.495\linewidth]{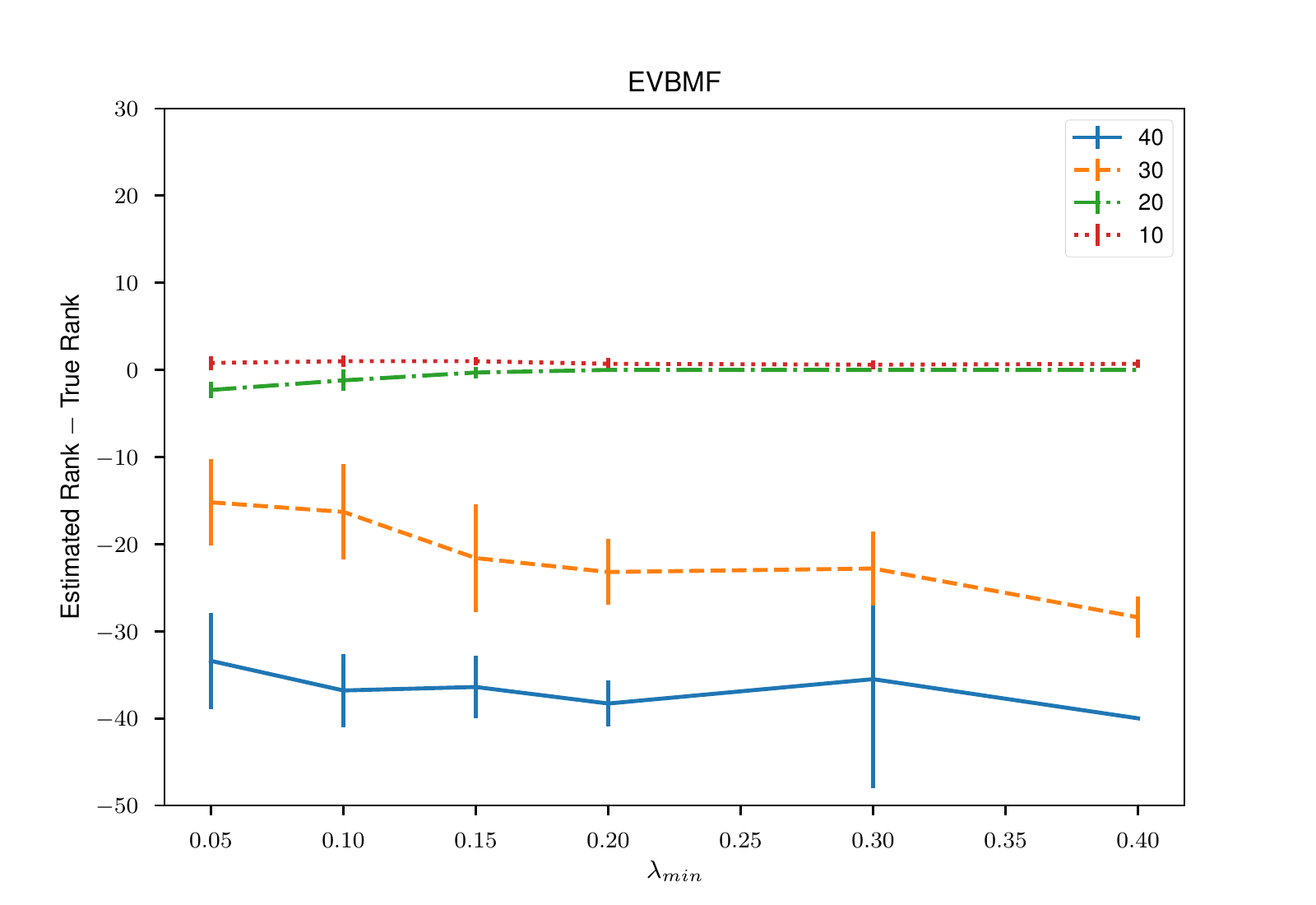}    
\includegraphics[width=0.495\linewidth]{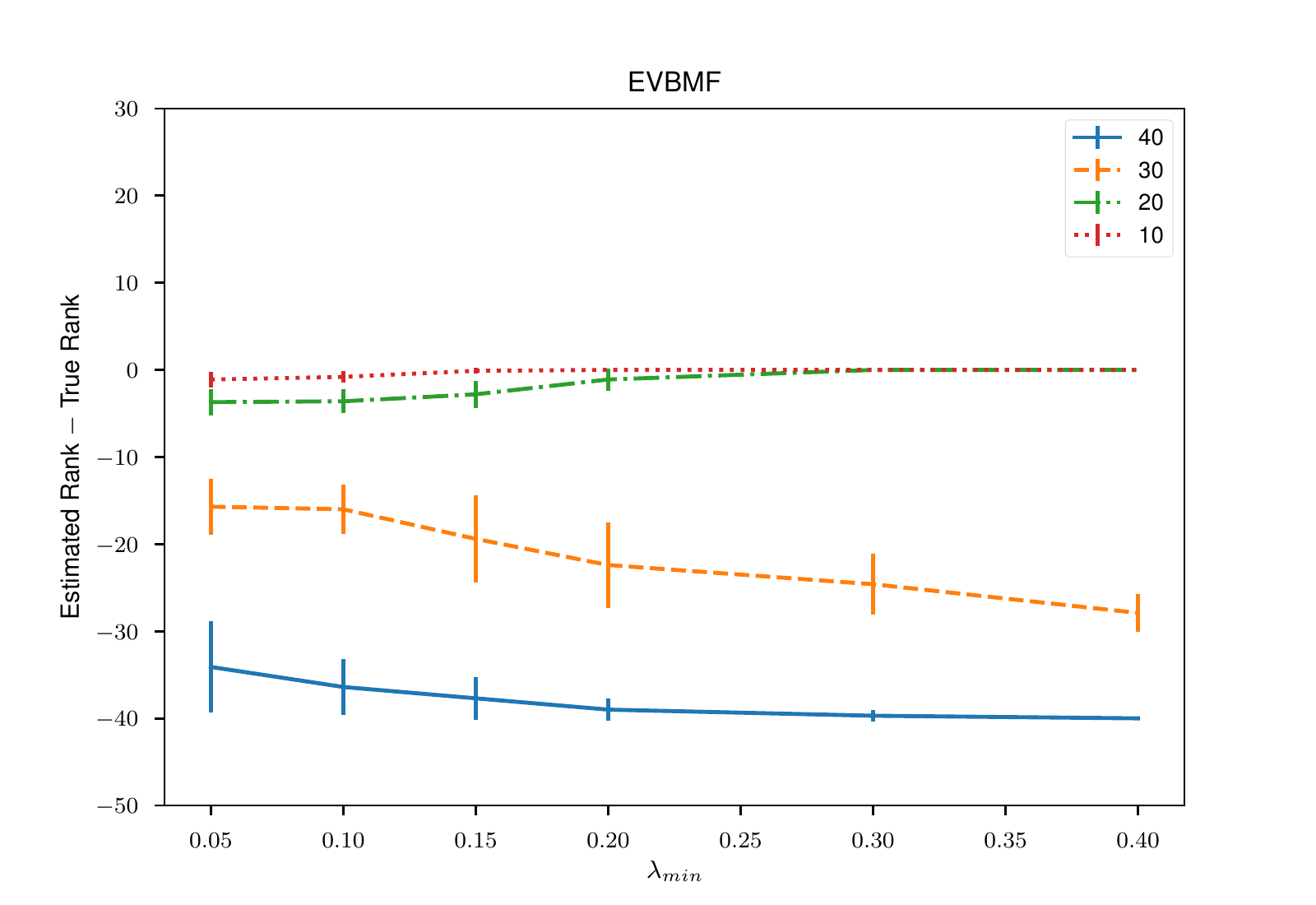}    
		\includegraphics[width=0.495\linewidth]{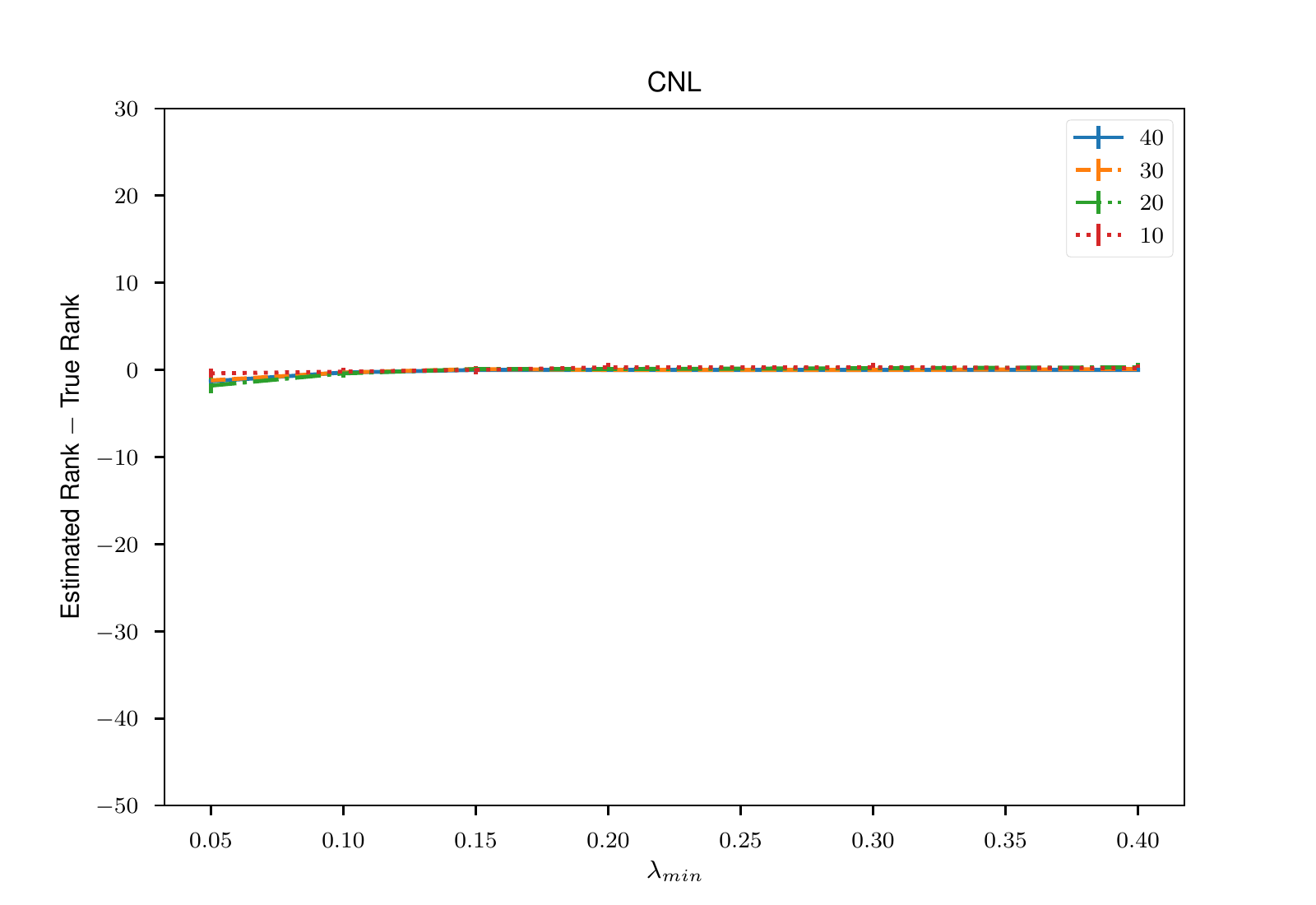}
		\includegraphics[width=0.495\linewidth]{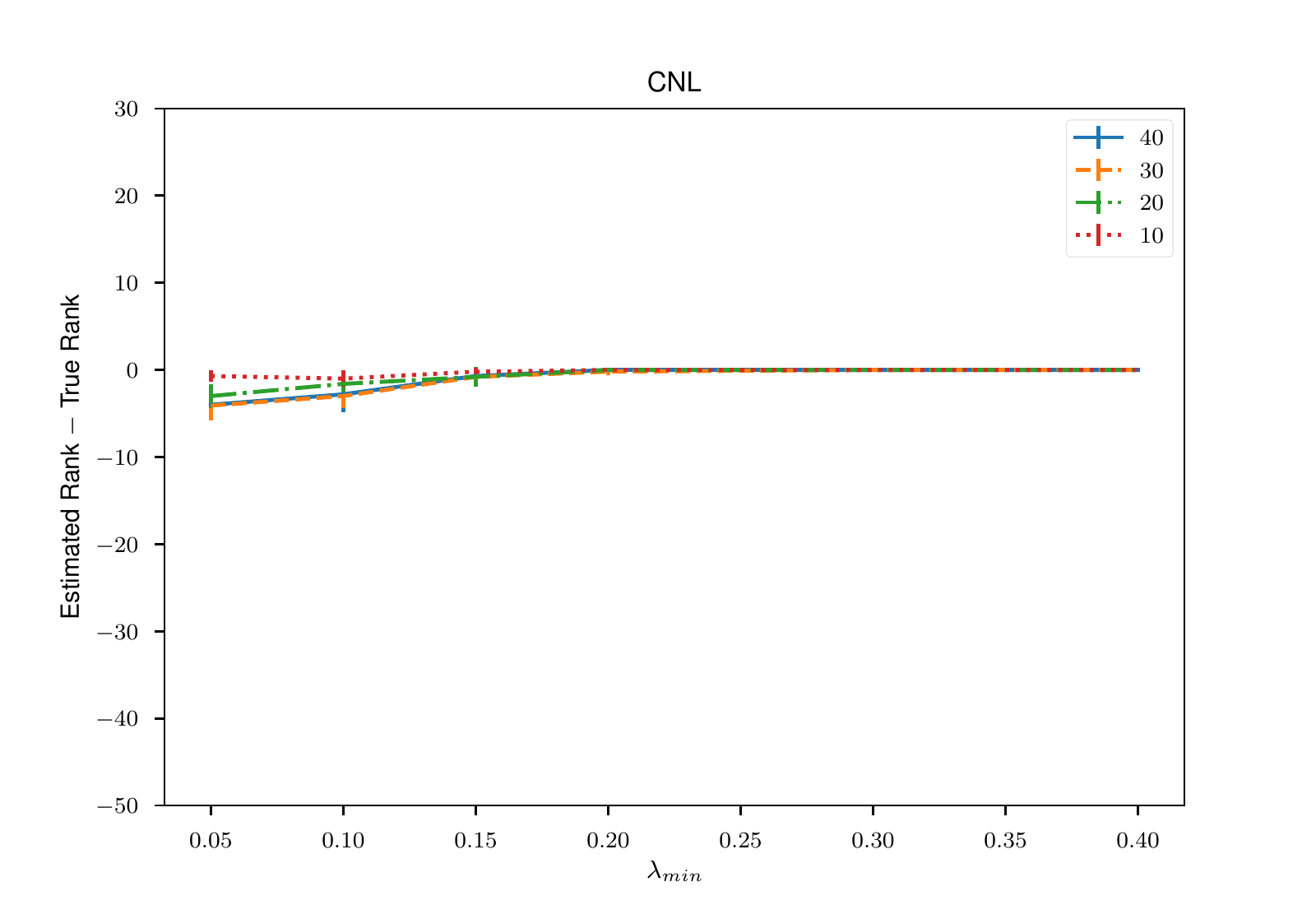}
	}
	\caption{Dimensionality recovery  by the baseline (upper figures), by EVBMF (center ones), and CNL minimization (lower ones). We set $(p,d)=(50,50)$ (resp.\,$(100,50)$) for left (resp.\,right) figures. The vertical axis represents the estimated value minus $d_\mr{true}$. The horizontal axis represents minimum non-zero singular values of $A_\mr{true}$. Each figure shows the average of $10$ experiments.  Each error bar represents the sample standard deviation.}
	\label{fig:rank_recovery}    
\end{figure}

We compare our method with a baseline method which consists of the following steps.
First, fix $\delta >0$. Second compute eigenvalues $\{ \lambda_1, \dots, \lambda_d\}$ of  an observed sample matrix. Lastly, estimate the rank of the signal part as  $ \#\{ j \mid \lambda_j  > \delta \}$. 
In our experiment, we chose a same value $\delta = 0.1$ for all cases, based on the estimation results in the case $p=d=50$ and $d_\mr{true} = 40$.

We also compare our method with the dimensionality recovery by 
the empirical variational Bayesian matrix factorization (EVBMF, for short) \cite{nakajima2013global} \cite{nakajima2015condition}
whose analytic solution is given by \cite[Theorem~2]{nakajima2015condition}. 
We use this solution because it requires no tuning of hyperparameters, and it recovers the true rank asymptotically as the large scale limit under some assumptions \cite[Theorem~13, Theorem~15]{nakajima2015condition}.

Figure~\ref{fig:rank_recovery} shows the dimensionality recovery experiments.
The horizontal axis indicates $\lambda_\mr{min}$.
The vertical axis indicates the estimated rank minus the true rank $d_\mr{true}$.
We observed that the baseline method did not work for all choices of $d_\mr{true}$ under the similar setting of $\delta$.
However, our CNL based method recovered the true rank for all choices of $d_\mr{true}$ and $\lambda_\mr{min} > 0.15$ with the same $\xi$.
Lastly, the EVBMF recovered it if $d_\mr{true}$ was low.
We conclude that our method estimates the true rank well under a suitable setting of $\xi$.

\sskip
\bpara{Validation Loss}
Figure~\ref{fig:val-rank-recovery} shows the validation loss curves under $\gamma=10,40$, and $\lambda_\mr{min} = 0.1, 0.2,0.3$ by the optimization via CNL.
It simultaneously recovered true rank and decreased validation loss for larger $\lambda_\mr{min}$.
For smaller $\lambda_\mr{min}$, it estimated smaller rank than $d_\mr{true}$ and did not continue to decrease the validation loss.

\begin{figure}[htbp]
	\ifthenelse{\boolean{DVIPDF}}{}{
		\centering
		\includegraphics[width=0.495\linewidth]{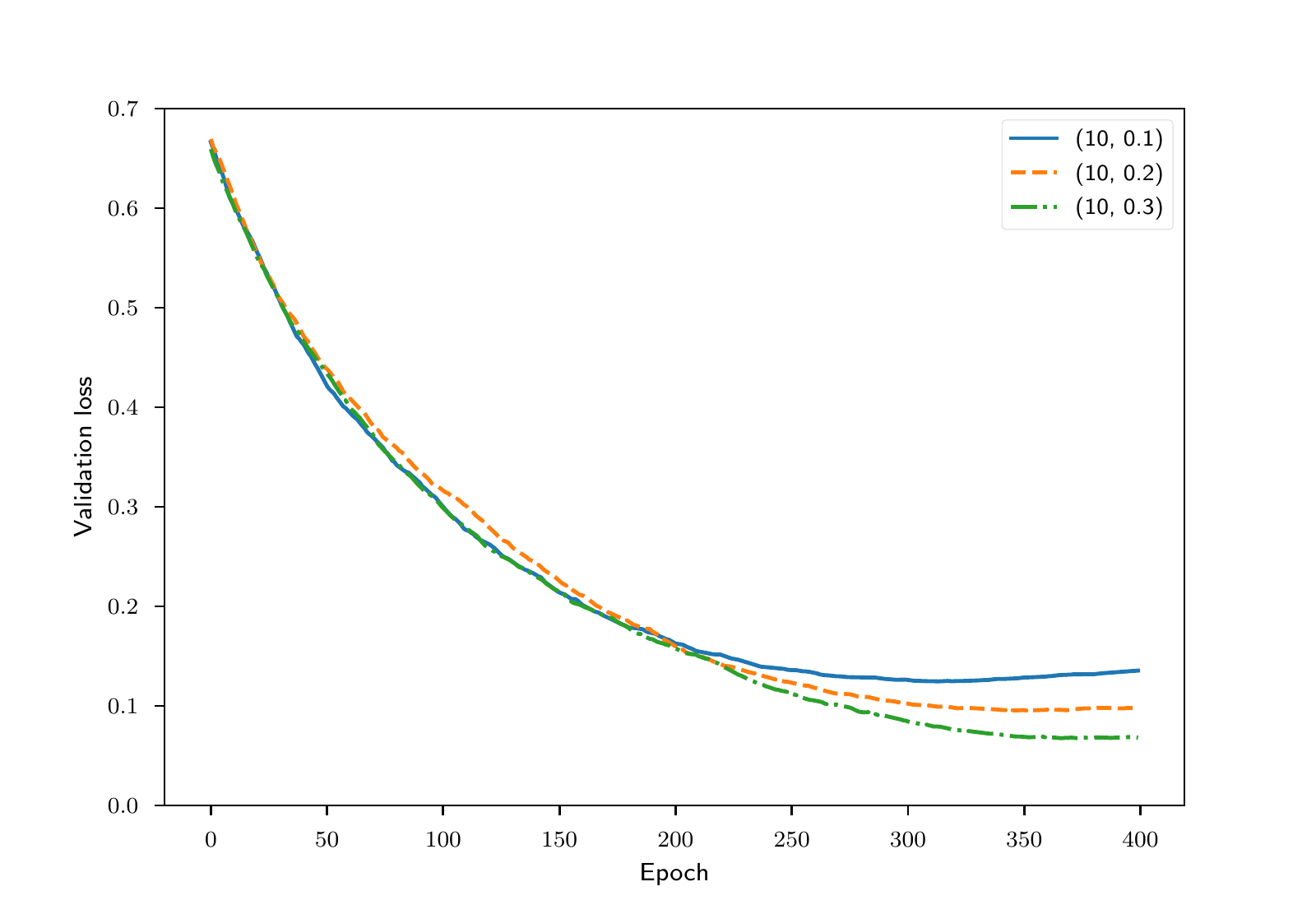}
		\includegraphics[width=0.495\linewidth]{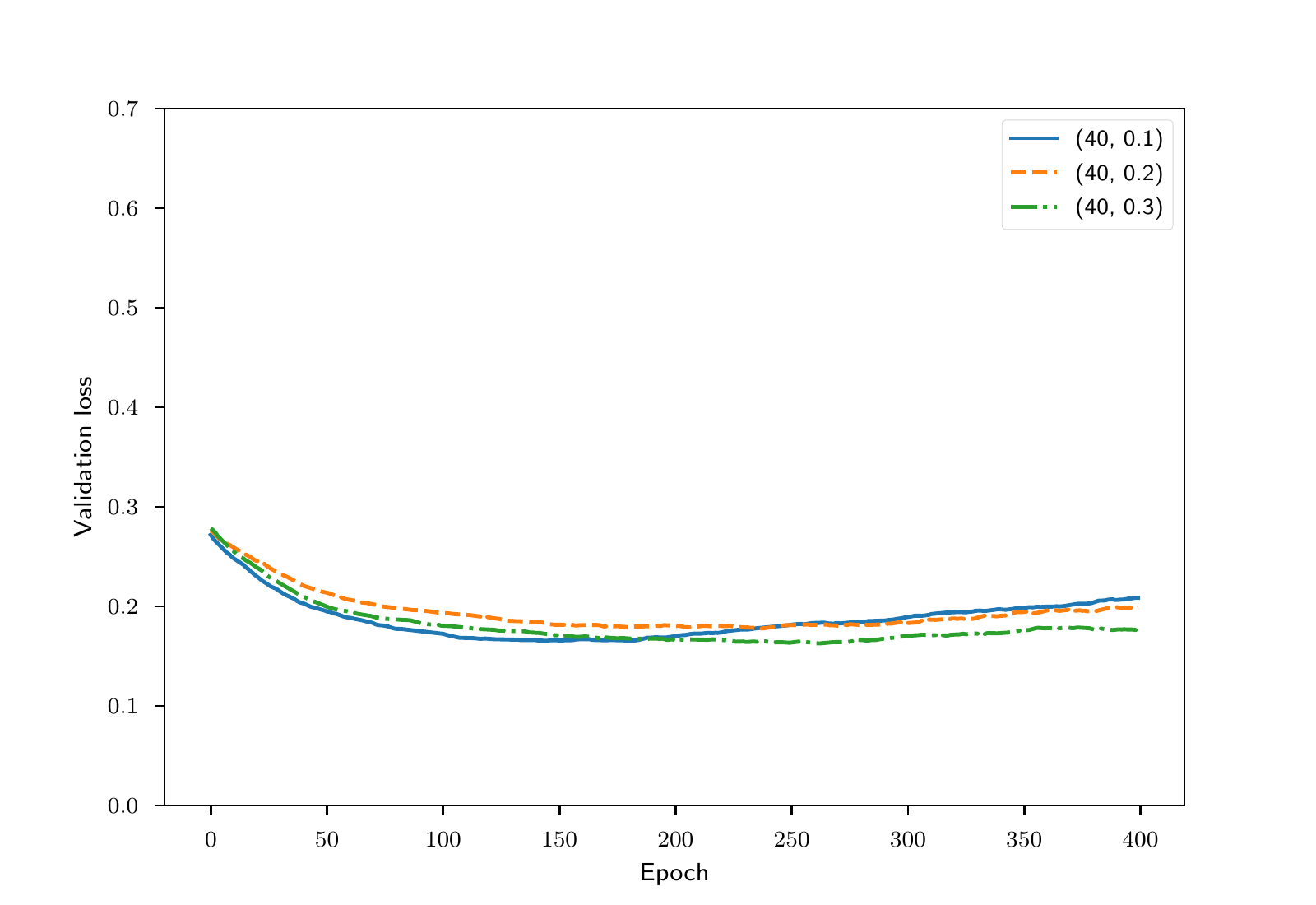}
	}
	\caption{Validations loss curves  with $d_\mr{true}=10$ (left) and that $d_\mr{true}=40$ (right). For each pair, the curves are average of $10$-experiments.}    
	\label{fig:val-rank-recovery}
\end{figure}

\bpara{Robustness to the Change of $\xi$}
The regularization coefficient $\xi$ affects the rank estimation in the same way as that in Lasso \cite{tibshirani1996regression}. 
Figure~\ref{fig:robustness} shows that the dimensionality recovery for the different choices of $\xi$.
We set $\xi_0 = 10^{-3}$ and $\xi = \xi_0/ 2,$ $ \xi_0,$  $2\xi_0 $.
To see  the effect of changing $\xi$, in this experiment we use a consitent threshold and the number of iterations $N=400d$, then count $\# \{ j \mid a_j > \xi_0 \}$ and use it to estimate $d_\mr{true}$.

\begin{figure}[htbp]
	\ifthenelse{\boolean{DVIPDF}}{}{
		\centering
		\includegraphics[width=0.495\linewidth]{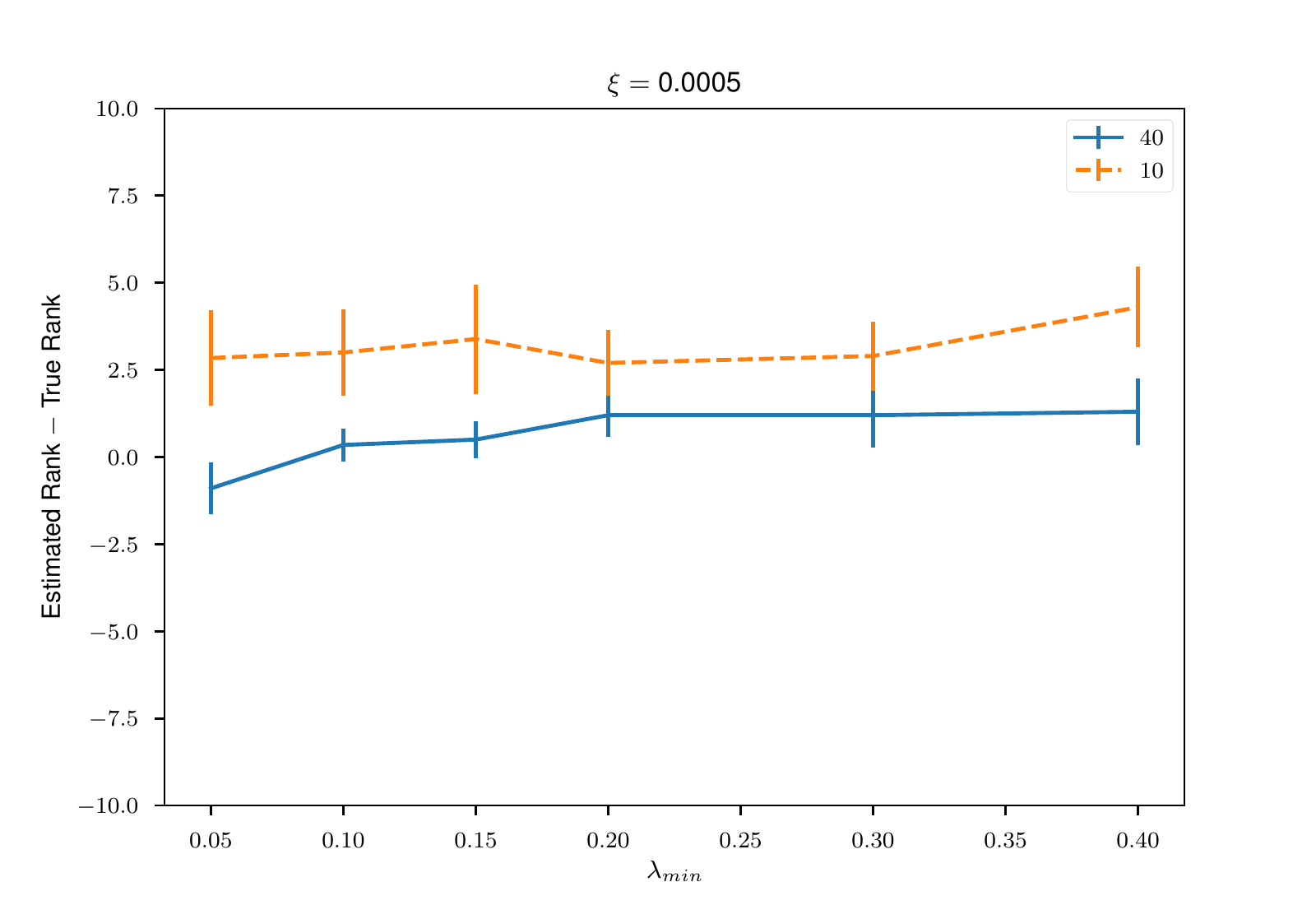}
		\includegraphics[width=0.495\linewidth]{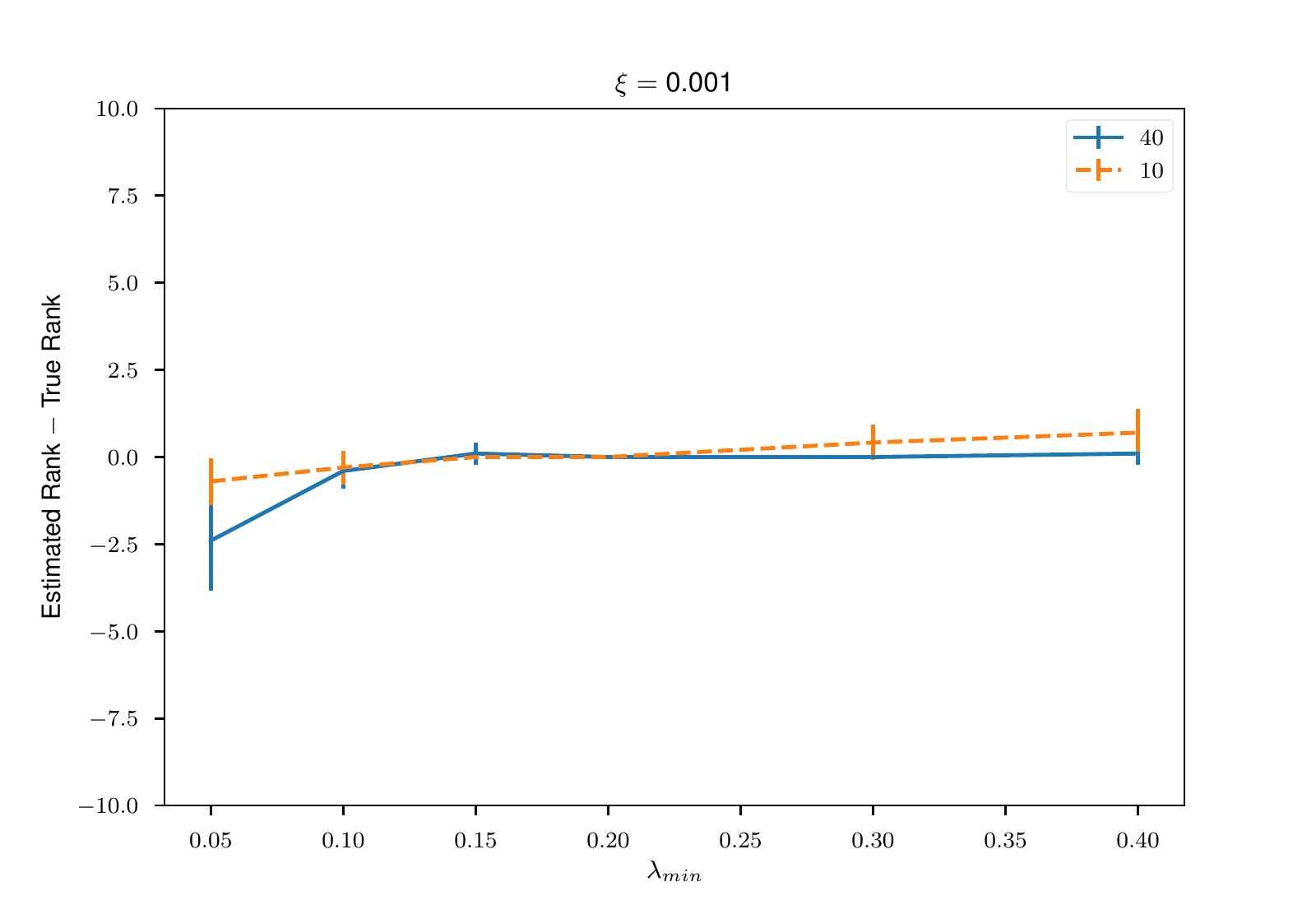}
		
		\includegraphics[width=0.495\linewidth]{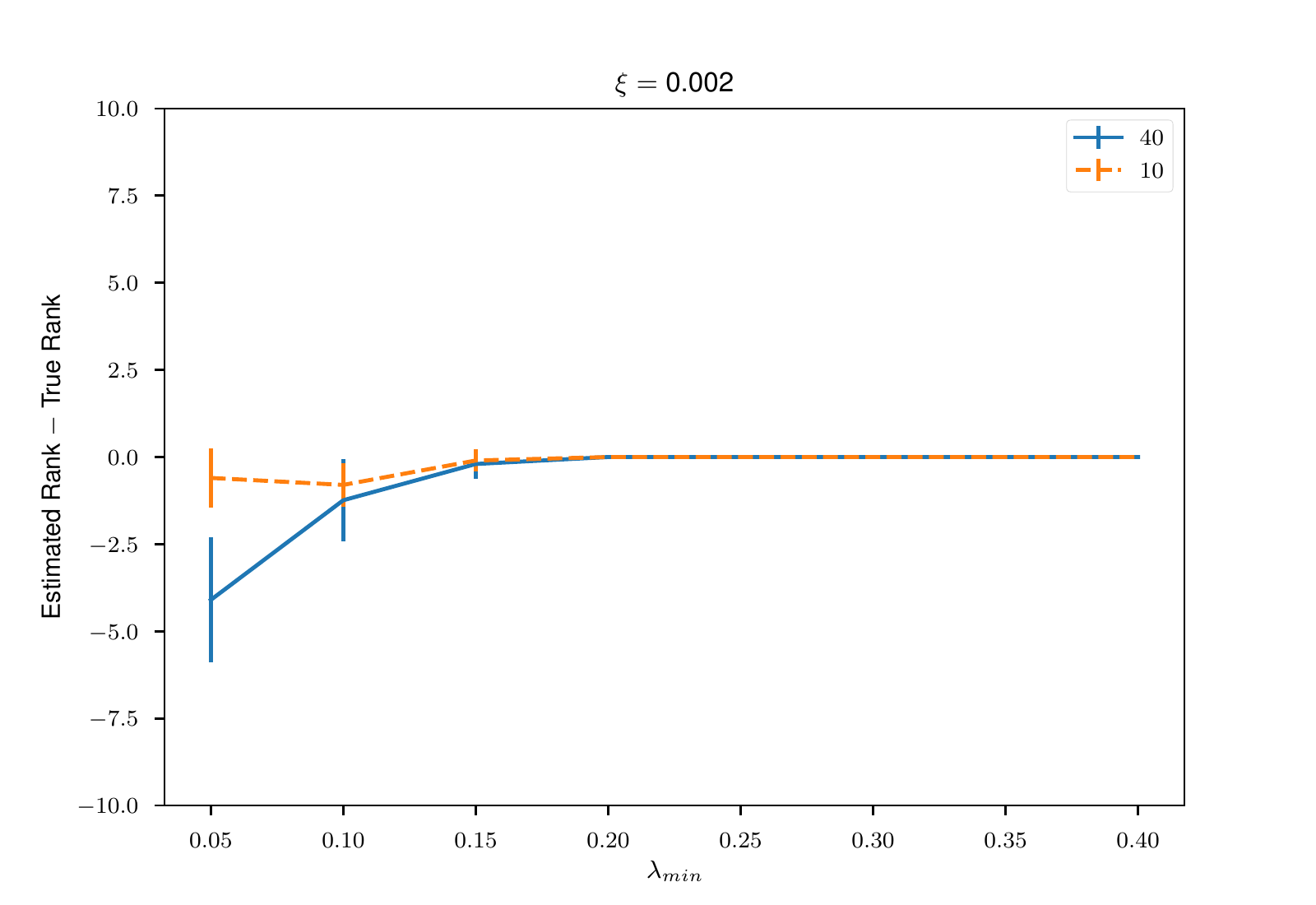}
	}
	\caption{Dimensionality Recovery by CNL under $\xi=5 \times 10^{-4}$ (upper left), $10^{-3}$ (upper right), and $2 \times10^{-3}$ (lower center). We set $d,p = 50$, $N=400d$, and $\gamma=10^{-1}$.
		Each error bar represents the sample standard deviation of $10$ experiments. Note that the range of vertical axis is changed from Figure~\ref{fig:rank_recovery}.}
	\label{fig:robustness}
\end{figure}

We observed that the small $\xi$ needed larger $N$ to recover the true rank, and a large $\xi$ gave unstable estimations for a small $\lambda_\mr{min}$.

Then further research is also needed to determine whether the $L^1$ regularization term is necessary or not to the dimensionality recovery.
More broadly, introducing cross-validation or a Bayesian framework of the Cauchy noise loss is also in the scope of future work.

\section{Conclusion}
\noindent This paper has introduced a new common framework of parameter estimation of random matrix models.
The framework is a combination of the Cauchy noise loss,   R-transform and the subordination, and online gradient descent.

Besides, we prove the determination gap converges uniformly to $0$ on each bounded parameter space. 
A vital point of the proof is that the integrand of the Cauchy cross-entropy has a bounded derivative.
Based on the theoretical observation of the Cauchy cross-entropy, we introduce an optimization algorithm of random matrix models. 
In experiments, it turned out that too small scale parameter $\gamma$ is deprecated.
Moreover, in the application to the dimensionality recovery,  our method surprisingly recovered the true rank even if the true rank was not low. 
However, it requires the setting of the weight of the $L^1$ regularization term.

This research has thrown up many questions in need of further investigation.
First, we need to find an ideal way to choose the scale parameter $\gamma$.  
A possible approach is to evaluate the variance of the determination gap.
Second, we need to prove the stability properties of Algorithm~\ref{alg:BOGD-FDE} because our loss function is non-convex.
Lastly, further research is also needed to determine whether the $L^1$ regularization term is necessary or not the dimensionality recovery.
More broadly, introducing cross-validation or an empirical Bayesian framework of Cauchy noise loss is also in the scope of future work.
\section*{Acknowledgement}
\noindent The author is greatly indebted to Roland Speicher for valuable discussion and several helpful comments about free probability theory. 
The author wishes to express his thanks to Genki Hosono for proving a lemma about complex analysis.
Kohei Chiba gives insightful comments and suggestions about probability theory.
The author also wishes to express his gratitude to Hiroaki Yoshida and Noriyoshi Sakuma for fruitful discussion.
The author is grateful for the travel support of Roland Speicher. 
Finally, the author gratefully appreciates the financial support of Benoit Collins that made it possible to complete this paper; our main idea was found during the workshop \enquote{Analysis in Quantum Information Theory}.

\bibliographystyle{abbrv}
\bibliography{reference}


\end{document}